\def\Statusstring{28 March 2026 \\ Technical Report}

\documentclass[12pt,letter]{article}

\usepackage{mytheoremenv}

\usepackage{url} 

\usepackage{tikz}
\usetikzlibrary{decorations.pathreplacing}
\usetikzlibrary{decorations.pathreplacing, shapes.geometric}
\usetikzlibrary{arrows.meta, positioning}

\usepackage{array}

\begin{document}

\title{A Tight Expressivity Hierarchy for GNN-Based Entity Resolution in Master Data Management}

\author{Ashwin~Ganesan%
  \thanks{Centre for Cybersecurity, Trust and Reliability (CyStar), Department of Computer Science and Engineering, Indian Institute of Technology, Madras, Tamilnadu, India. Email:
\texttt{ashwin.ganesan@gmail.com}}
}

\date{}
\maketitle

\vspace{-8.0cm}
\begin{flushright}
	\texttt{\Statusstring}\\[1cm]
\end{flushright}
\vspace{+5.0cm}

\begin{abstract}
	Entity resolution --- identifying database records that refer to the
	same real-world entity --- is naturally modelled on bipartite graphs
	connecting entity nodes to their attribute values.  A common approach
	is to apply a message-passing neural network (MPNN) enriched with all
	available architectural extensions (reverse message passing, port
	numbering, ego IDs), but this results in unnecessary computational
	overhead, since different entity resolution tasks have fundamentally
	different complexity.  A natural question is therefore: for a given
	matching criterion, what is the cheapest MPNN architecture that
	provably works?
	
	In this work, we answer this question with a four-theorem separation
	theory on typed entity-attribute graphs.  We introduce a family of
	co-reference predicates, $\mathrm{Dup}_r$, capturing the evidence
	pattern ``two same-type entities share at least $r$ attribute
	values,'' together with the $\ell$-cycle participation predicate
	$\mathrm{Cyc}_\ell$ for settings with entity--entity edges.  For
	the co-reference predicates, we prove tight lower and upper
	bounds --- constructing graph pairs that are provably
	indistinguishable by every MPNN lacking the required adaptation,
	and exhibiting explicit, minimal-depth MPNNs that compute the
	predicate correctly on all inputs.  For cycle detection, we prove
	the necessity of ego IDs and demonstrate their sufficiency on the
	canonical separation instances.
	
	The central finding is a sharp complexity gap between detecting
	\emph{any} shared attribute and detecting \emph{multiple} shared
	attributes.  The former is a purely local computation: each attribute
	can independently check whether two same-type entities point to it,
	requiring only reverse message passing in two layers.  The latter
	demands what we call \emph{cross-attribute identity correlation} ---
	verifying that the \emph{same} entity appears at several attributes
	of the target --- a fundamentally non-local requirement that
	necessitates ego IDs and four layers, even on the simplest class of
	acyclic bipartite graphs.  A similar necessity holds for cycle
	detection.  Taken together, these results yield a
	minimal-architecture principle: practitioners can select the cheapest
	adaptation set that provably suffices for their specific matching
	criterion, with a guarantee that no simpler architecture works.
	Computational validation confirms every theoretical prediction.
\end{abstract}


\tableofcontents


\section{Introduction}

Entity resolution --- the problem of determining whether two or more
database records refer to the same real-world entity --- is a
foundational task in data integration, knowledge graph construction,
and master data management.  In large-scale commercial databases such
as those maintained by business-data providers, entity resolution
operates over hundreds of millions of records with richly typed
attributes: names, addresses, phone numbers, email addresses, tax
identifiers, and corporate-ownership relationships, each carrying its
own type and semantics.  Errors in entity resolution propagate
downstream into customer analytics, regulatory reporting, and fraud
detection, making both precision and scalability critical requirements.

A natural formalism for this setting is the \emph{typed
	entity-attribute graph}: entity nodes represent database records,
attribute nodes represent canonical attribute values, and typed
directed edges connect each entity to its attributes.  Two entities
that share many typed attribute values --- the same email address, the
same phone number, and the same mailing address --- are likely
co-referent.  When entity--entity edges are also present, as in
transaction networks or corporate-ownership graphs, directed cycles
become important signals for fraud detection and regulatory
compliance.

Graph neural networks (GNNs), and specifically message-passing neural
networks (MPNNs)~\cite{Scarselli:2009, Gori:2005}, have been
increasingly adopted for entity resolution on such
graphs~\cite{Li:GraphER:AAAI:2020, Yao:HierGAT:SIGMOD:2022,
	Ganesan:IBM:2020}, since entity-attribute data is inherently
relational and practical entity matching requires soft similarity
judgments over noisy, partially observed attribute values --- a
setting where exact combinatorial methods such as hash-based joins
are insufficient.  While the exact co-reference predicates
$\mathrm{Dup}_r$ studied in this paper admit efficient
combinatorial solutions on clean data, they arise as the
\emph{noiseless limit} of the soft matching functions that practical
GNN-based systems learn; the motivation for this computational model
is developed in Section~\ref{sec:motivation}.  However, standard
MPNNs have well-known expressivity limitations: their distinguishing
power is bounded by the 1-Weisfeiler--Leman graph isomorphism
test~\cite{Weisfeiler:1968, Shervashidze:2011,
	Xu-et-al:ICLR:2019}.  To overcome these limitations on directed
multigraphs, Egressy et al.~\cite{Egressy-et-al:AAAI:2024} propose
three architectural adaptations --- reverse message passing, directed
multigraph port numbering, and ego IDs --- and prove that the
combination of all three enables detection of any directed subgraph
pattern.

For entity resolution, however, this universality result is both more
and less than what practitioners need.  It is \emph{more} because it
guarantees detection of arbitrary subgraph patterns, whereas entity
resolution requires only specific patterns: shared attributes
between same-type entities, and directed cycles.  It is \emph{less}
because it does not indicate which adaptations can be safely omitted
for a given task --- a question with direct computational
consequences, since ego IDs alone increase the per-layer cost from
$O(|E|)$ to $O(|V| \cdot |E|)$.  This raises the central question
of the present work: \emph{which subsets of the three adaptations are
	necessary and sufficient for each entity resolution task?}

We answer this question with a four-theorem separation theory on
typed entity-attribute graphs.  We formalise a family of co-reference
predicates $\mathrm{Dup}_r$ capturing the evidence pattern ``two
entities of the same type share at least $r$ typed attribute values,''
together with the $\ell$-cycle participation predicate
$\mathrm{Cyc}_\ell$ for settings with entity--entity edges.  For each
predicate and each graph class, we prove both a \emph{necessity}
result --- constructing pairs of graphs that are provably
indistinguishable by every MPNN equipped with the weaker adaptation
set --- and a \emph{sufficiency} result --- exhibiting an explicit
MPNN of minimal depth that computes the predicate correctly on all
graphs in the class.  The results are summarised in
Table~\ref{tab:minimal-architecture}, which gives a complete
\emph{minimal-architecture principle} for GNN-based entity resolution:
for each co-reference task, the cheapest sufficient adaptation set
and the optimal depth, together with a proof that every strictly
cheaper alternative fails; for cycle detection, the necessary
adaptations and those sufficient on the canonical separation
instances.  The theoretical predictions are validated
computationally in Section~\ref{sec:computational}.

The central finding is a sharp complexity gap between
$\mathrm{Dup}_1$ and $\mathrm{Dup}_r$ for $r \geq 2$.  Detecting
whether an entity shares \emph{any} attribute value with a same-type
competitor is a purely local computation requiring only reverse
message passing in two layers.  Detecting whether it shares
\emph{at least $r$} attribute values additionally requires ego IDs
and four layers, even on simple, acyclic, bipartite typed
entity-attribute graphs.  The underlying obstacle is what we term
\emph{cross-attribute identity correlation}: verifying that the same
entity appears at multiple attributes of the target entity is a
fundamentally non-local requirement that no MPNN without ego IDs can
satisfy.

\medskip
\noindent\textbf{Contributions.}
The contributions of this paper are as follows.
We introduce the \emph{typed entity-attribute graph} formalism and the
family of co-reference predicates $\mathrm{Dup}_r$ that capture the
core computational patterns in entity resolution, along with the
$\ell$-cycle participation predicate $\mathrm{Cyc}_\ell$ for settings
with entity--entity edges.
We prove four tight separation theorems.  Theorem~\ref{thm:K21-simple}
shows that reverse message passing is necessary and sufficient for
$K_{2,1}$ detection on simple typed entity-attribute graphs, with
optimal depth~$2$.  Theorem~\ref{thm:K21-multigraph} shows that
incoming port numbering is additionally necessary and sufficient on
multigraph typed entity-attribute graphs, again with optimal
depth~$2$.  Theorem~\ref{thm:K2r-simple}, our main result, shows
that for $r \geq 2$, ego IDs are necessary and sufficient together
with reverse message passing for $K_{2,r}$ detection on simple typed
entity-attribute graphs, with optimal depth~$4$.  This establishes a
strict complexity gap between $\mathrm{Dup}_1$ and $\mathrm{Dup}_r$.
Theorem~\ref{thm:cycle-detection} shows that ego IDs are necessary
for $\ell$-cycle detection on typed directed graphs, and sufficient
alone on the canonical separation instances, with optimal
depth~$\ell$; universal $\mathrm{Cyc}_\ell$ computation on all
typed directed multigraphs requires the full triple of
adaptations~\cite[Corollary~4.4.1]{Egressy-et-al:AAAI:2024}.
Each sufficiency proof is constructive, exhibiting an explicit MPNN
with the stated depth; each necessity proof constructs concrete
separation graphs with an indistinguishability argument valid for
MPNNs of arbitrary depth and width.  Together, these results yield the
minimal-architecture principle of
Table~\ref{tab:minimal-architecture}, enabling practitioners to select
the cheapest MPNN adaptation that provably suffices for their specific
entity resolution task.

\section{Problem Formulation} \label{sec:prob-formulation} 

	A \emph{typed directed multigraph} is a tuple $G = (V, E, \mathrm{src}, \mathrm{dst}, \mathcal{T}_V, \mathcal{T}_E, \tau_V, \tau_E)$, where
	$V$ is a finite set of nodes,
	$E$ is a finite set of abstract edge identifiers,
	$\mathrm{src} \colon E \to V$ and $\mathrm{dst} \colon E \to V$ assign to each edge its source and destination,
	$\mathcal{T}_V$ is a set of node types,
	$\mathcal{T}_E$ is a set of edge types,
	$\tau_V \colon V \to \mathcal{T}_V$ is a node type function, and
	$\tau_E \colon E \to \mathcal{T}_E$ is an edge type function.
	Parallel edges are permitted: distinct edges $e \neq e'$ may share the same source, destination, and type. 
	For brevity, we denote an edge $e$ by the triple $(\mathrm{src}(e),\, \mathrm{dst}(e),\, \tau_E(e))$, written $(u, v, \tau)$, with the understanding that multiple edges may share the same triple.

The set of \emph{predecessors} (in-neighbors) of a node $v \in V$ and the set of \emph{successors} (out-neighbors) of a node $v \in V$ are denoted $N_{\mathrm{in}}(v)$ and $N_{\mathrm{out}}(v)$, respectively:
\begin{align*}
	N_{\mathrm{in}}(v) &= \{u \in V : \exists\, \tau \in \mathcal{T}_E,\; (u,v,\tau) \in E\}, \\
	N_{\mathrm{out}}(v) &= \{w \in V : \exists\, \tau \in \mathcal{T}_E,\; (v,w,\tau) \in E\}.
\end{align*}

\begin{Definition}
A \emph{typed entity-attribute graph} is a typed directed multigraph $G=(V,E, \mathrm{src}, \mathrm{dst}, \mathcal{T}_V, \mathcal{T}_E, \tau_V, \tau_E)$ equipped with a bipartition $V = \mathcal{E} \dot{\cup} \mathcal{A}$, where $\mathcal{E}$ is a set of entity nodes, $\mathcal{A}$ is a set of attribute nodes, every edge is directed from an entity node to an attribute node, and the set $\mathcal{T}_V$ of node types decomposes as $\mathcal{T}_V = \mathcal{T}_\mathcal{E} \dot{\cup} \mathcal{T}_\mathcal{A}$, where $\mathcal{T}_\mathcal{E}$ is a set of entity types, and $\mathcal{T}_\mathcal{A}$ is a set of attribute types.  Thus, for all $(u,a,\tau) \in E$, $u \in \mathcal{E}$ and $a \in \mathcal{A}$; the type of an entity node $u \in \mathcal{E}$ is $\tau_V(u) \in \mathcal{T}_\mathcal{E}$, and the type of an attribute node $a \in \mathcal{A}$ is $\tau_V(a) \in \mathcal{T}_\mathcal{A}$. 

A typed entity-attribute graph is \emph{simple} if for each entity-attribute pair $(u,a) \in \mathcal{E} \times \mathcal{A}$ and each edge type $\tau \in \mathcal{T}_E$, there is at most one edge from $u$ to $a$ of type $\tau$.  
\end{Definition}

A typed entity-attribute graph is essentially a typed directed bipartite multigraph;
see Figure~\ref{fig:ea-graph-example} for an example with two entity nodes -- both of the same type (Person) -- and three attribute nodes.

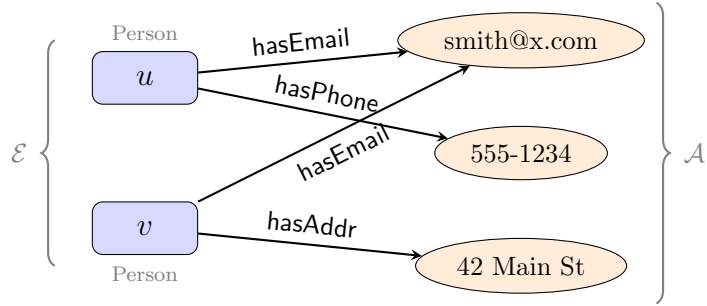
\begin{figure}[h]
	\centering
	\begin{tikzpicture}[
		entity/.style={draw, rectangle, rounded corners, minimum width=1.4cm, minimum height=0.7cm, fill=blue!15},
		attribute/.style={draw, ellipse, minimum width=1.2cm, minimum height=0.7cm, fill=orange!15},
		edge label/.style={font=\footnotesize, midway, above, sloped},
		>=stealth
		]
		\node[entity] (u) at (0, 2) {$u$};
		\node[entity] (v) at (0, 0) {$v$};
		
		\node[attribute] (email) at (5, 2.5) {\footnotesize smith@x.com};
		\node[attribute] (phone1) at (5, 1) {\footnotesize 555-1234};
		\node[attribute] (addr) at (5, -0.5) {\footnotesize 42 Main St};
		
		\draw[->, thick] (u) -- node[edge label] {\textsf{hasEmail}} (email);
		\draw[->, thick] (u) -- node[edge label] {\textsf{hasPhone}} (phone1);
		
		\draw[->, thick] (v) -- node[edge label, below, sloped] {\textsf{hasEmail}} (email);
		\draw[->, thick] (v) -- node[edge label] {\textsf{hasAddr}} (addr);
		
		\node[font=\scriptsize, gray] at (0, 2.6) {Person};
		\node[font=\scriptsize, gray] at (0, -0.6) {Person};
		
		\draw[decorate, decoration={brace, amplitude=5pt}, thick, gray] (-1.2, -0.5) -- (-1.2, 2.5) node[midway, left=6pt, gray, font=\footnotesize] {$\mathcal{E}$};
		\draw[decorate, decoration={brace, amplitude=5pt, mirror}, thick, gray] (6.8, -1) -- (6.8, 3) node[midway, right=6pt, gray, font=\footnotesize] {$\mathcal{A}$};
	\end{tikzpicture}
	\caption{A typed entity-attribute graph. Entities $u$ and $v$ share a common attribute node \textsf{smith@x.com} via edge of type \textsf{hasEmail}.}
	\label{fig:ea-graph-example}
\end{figure}

\subsection{Duplicate Entity Detection as Binary Node Classification}

We consider two entities to be potential duplicates if they share a common attribute value. We say entity $u$ has a duplicate, denoted by $\mathrm{Dup}(u)=1$, if there exists another entity $v$ of the same type as $u$ such that $u$ and $v$ share a common attribute via edges of the same type. For example, in Figure~\ref{fig:ea-graph-example}, node $u$ has a duplicate because there exists another entity, namely $v$, such that $u$ and $v$ are of the same type (Person) and are both connected to the attribute node  \textsf{smith@x.com} via edges of type \textsf{hasEmail}.  Thus, for this typed entity-attribute graph, $\mathrm{Dup}(u)=1$.  

More formally, for an entity $u \in \mathcal{E}$, define the typed neighborhood
\[
N_{\mathrm{typed}}(u) = \bigl\{(a, \tau) \in \mathcal{A} \times \mathcal{T}_E: (u,a,\tau) \in E \bigr\}.
\]

Given an input typed entity-attribute graph, the \emph{$K_{2,r}$ detection problem} for $r \ge 1$ is the following binary node classification task: for each entity $u \in \mathcal{E}$, compute
\[
\mathrm{Dup}_r(u) = \begin{cases} 
	1 & \text{if } \exists v \in \mathcal{E} \setminus \{u\}, \; \tau_V(v) = \tau_V(u), \; \bigl| N_{\mathrm{typed}}(u) \cap N_{\mathrm{typed}}(v) \bigr| \ge r \\ 
	0 & \text{otherwise}.
\end{cases}
\]
The same-type constraint $\tau_V(u) = \tau_V(v)$ reflects that in entity resolution, duplicates are only meaningful among entities of the same type, such as two Person records, or two Organization records, not a Person and an Organization. We write $\mathrm{Dup}$ for $\mathrm{Dup}_1$.  When the input graph must be made explicit---as in separation arguments where two graphs $G_1$ and $G_2$ are compared---we append it as a further subscript, writing $\mathrm{Dup}_{r,G}(u)$; this is abbreviated to $\mathrm{Dup}_G(u)$ when
$r = 1$ is clear from context.

For example, in Figure~\ref{fig:ea-graph-example}, $\mathrm{Dup}(u)=\mathrm{Dup}_1(u)=1$ since $u$ and $v$ share a common typed attribute, but $\mathrm{Dup}_2(u)=0$ because there does not exist another entity $v \ne u$ that shares \emph{two} attributes in common with $u$.  

When two entities $u$ and $v$ of the same type share $r$ typed neighbors in common, they share $r$ attribute values -- potentially of $r$ distinct types. For example, if Person entities $u$ and $v$ both connect to the same email node $a_1$ via edge type \textsf{hasEmail}, the same phone node $a_2$ via edge type \textsf{hasPhone}, and the same address node $a_3$ via edge type \textsf{hasAddr}, then $|N_{\mathrm{typed}}(u) \cap N_{\mathrm{typed}}(v)| \ge 3$, whence $\mathrm{Dup}_3(u)=1$. Crucially, the $r$ shared attributes may involve $r$ distinct attribute types and $r$ distinct edge types. 

The parameter $r$ controls the strength of evidence required to flag a potential
duplicate. In practice, $r=1$ captures \emph{weak} matches: two Person records
sharing just a phone number, or just an email address, may be coincidental
--- for instance, family members sharing a phone, or a generic company email. Requiring
$r \geq 2$ captures \emph{strong} matches: two Person records sharing both the
same email address and the same phone number are far more likely to be genuine
duplicates. As we shall see, this seemingly
modest increase from $r=1$ to $r \geq 2$ requires a strictly more powerful GNN architecture.

\subsection{Cycle Detection}

In some MDM settings, entity nodes are connected not only to attribute nodes 
but also to other entity nodes. For example, in financial crime detection, 
Person entities may be linked by \textsf{transfersTo} edges representing 
fund transfers, or Company entities may be linked by \textsf{owns} edges 
representing ownership relationships. Cyclic patterns in such graphs ---
e.g., circular fund transfers or layered ownership loops --- are important 
signals for fraud detection and regulatory compliance. This motivates the 
following cycle detection task on typed directed graphs.

For typed directed graphs that contain entity-entity edges, define the
\emph{$\ell$-cycle detection problem} ($\ell \ge 3$) as the following
binary node classification task: for each node $v \in V$, compute
\[
\mathrm{Cyc}_\ell(v) = \begin{cases}
	1 & \text{if } v \text{ belongs to a directed \emph{simple} cycle of
		length exactly } \ell \text{ in } G \\
	0 & \text{otherwise.}
\end{cases}
\]
That is, $\mathrm{Cyc}_\ell(v) = 1$ if and only if there exist distinct
vertices $v = v_0, v_1, \ldots, v_{\ell-1}$ such that, for each
$i \in \{0, \ldots, \ell-1\}$, there exists an edge type
$\tau^{(i)} \in \mathcal{T}_E$ with
$(v_i,\, v_{i+1 \bmod \ell},\, \tau^{(i)}) \in E$.

Note that the edge types $\tau^{(0)}, \ldots, \tau^{(\ell-1)}$ along the
cycle need not be identical: the definition permits
\emph{type-heterogeneous} cycles.  In some applications---such as
circular fund transfers in a transaction network---every edge in the
cycle carries the same type, such as \textsf{transfersTo}, yielding a
type-homogeneous cycle.  In others---such as layered beneficial-ownership
structures that alternate \textsf{directlyOwns} and
\textsf{beneficiallyOwns} edges---cycles naturally mix edge types.  We
adopt the more general, type-agnostic definition because it subsumes the
homogeneous case: a practitioner who wishes to detect cycles along a
single edge type~$\tau$ may simply restrict the input graph to the
subgraph induced by edges of type~$\tau$ and then apply $\mathrm{Cyc}_\ell$.

\subsection{Research Question}

A standard message-passing neural network (MPNN) can be enhanced with adaptations such as reverse message passing, directed multigraph port numbering, and ego IDs (defined in Section~\ref{sec:preliminaries}), to yield a universal architecture capable of detecting any directed subgraph pattern \cite{Egressy-et-al:AAAI:2024}.  In the present work, we study which subsets of these adaptations are necessary or sufficient for computing $\mathrm{Dup}_1$, $\mathrm{Dup}_r$, and $\mathrm{Cyc}_\ell$ on typed directed multigraphs arising in entity resolution.  We prove four separation theorems that together yield a minimal-architecture principle for entity resolution.

\textbf{Notation}.  Throughout, we write directed edges as triples $(u,a,\tau)$, which denotes an edge from node $u$ to node $a$ of edge type $\tau$. In a typed entity-attribute graph, it is understood that the source $u \in \mathcal{E}$ and the destination $a \in \mathcal{A}$.  Entity nodes are typically represented by variables $u,v,w,x$, and attribute nodes are typically represented using variables $a_1,a_2,\ldots$. Specific entity types are denoted $\sigma \in \mathcal{T}_\mathcal{E}$, attribute type values $\alpha_j \in \mathcal{T}_\mathcal{A}$, and edge type values $\tau_j \in \mathcal{T}_E$ with numeric subscripts.  Recall that the functions $\tau_V$ and $\tau_E$ denote the node and edge type maps, respectively.    

We use $\hat{y}(v)$ to denote the output of the MPNN's readout function at node $v$, and reserve $\mathrm{Dup}(v)$, $\mathrm{Dup}_r(v)$ and $\mathrm{Cyc}_\ell(v)$ to denote the ground-truth predicates defined above.  Thus, a correct MPNN achieves $\hat{y}(v) = f(v)$ for every node $v$ and every input graph $G$, where $f$ is the target predicate.

\subsection{Motivation: Why GNNs for Entity Resolution}
\label{sec:motivation}

The co-reference predicates $\mathrm{Dup}_r$ defined above ask
whether two entities of the same type share at least~$r$ exact typed
attribute values.  On clean data with canonical attribute strings,
this question admits efficient combinatorial solutions: the $K_{2,1}$
predicate reduces to grouping entities by shared attribute nodes via
a hash index in $O(|E|)$ time, and $K_{2,r}$ for general $r$
reduces to counting shared attributes per entity pair.  No neural
network is required for exact matching on clean data.  We study these
predicates within the GNN framework for three reasons, which together
justify the relevance of the expressivity results to the broader
entity resolution literature.

\medskip
\noindent\textbf{Soft matching on noisy data.}
The primary reason that GNNs are an appropriate computational model
for entity resolution is that real-world attribute values are noisy.
In large-scale commercial databases, co-referent records rarely share
identical attribute strings: name variations such as ``John Smith''
versus ``J.\ Smith'' versus ``John F.\ Smith,'' formatting
differences in addresses and phone numbers, and data entry errors are
pervasive.  A hash-based join requires two attribute strings to be
identical in order to detect a shared value; any discrepancy, no
matter how minor, causes the match to be missed entirely.

A GNN operating on learned attribute embeddings avoids this
brittleness.  Rather than testing string identity, the GNN computes
similarity in a continuous embedding space where nearby points
correspond to similar attribute values.  Through training on labelled
examples of known duplicates and known non-duplicates, the GNN learns
an embedding function that maps ``John Smith'' and ``John F.\ Smith''
to nearby vectors while mapping ``John Smith'' and ``Maria Garcia''
to distant vectors.  The message-passing architecture then aggregates
these soft similarity signals across multiple attributes, producing a
co-reference score that degrades gracefully with noise rather than
failing catastrophically.

To make this precise, let $\psi \colon \mathcal{A} \to \mathbb{R}^d$
be a feature encoder that maps each attribute node to a vector
representation based on its attribute string --- for instance, via a
character-level neural network or a pre-trained language model ---
and let $\mathrm{sim} \colon \mathbb{R}^d \times \mathbb{R}^d \to
[0,1]$ be a similarity function.  In the MPNN framework of
Section~\ref{sec:preliminaries}, the role of~$\psi$ is absorbed
into the initial embedding~$h^{(0)}(a)$ of each attribute
node~$a$.  Define the \emph{soft overlap} between two entities
$u, v \in \mathcal{E}$ by
\[
\mathrm{ov}_{\mathrm{soft}}(u, v) \;=\;
\sum_{(a,\,\tau)\, \in\, N_{\mathrm{typed}}(u)}
\;\max_{\substack{(a',\,\tau')\, \in\, N_{\mathrm{typed}}(v)
		\\[1pt] \tau' \,=\, \tau}}
\mathrm{sim}\bigl(\psi(a),\, \psi(a')\bigr),
\]
where the maximum over an empty set is taken to be~$0$, and the
constraint $\tau' = \tau$ restricts the comparison to attribute pairs
connected via the same edge type.  Note that, like the exact overlap
$|N_{\mathrm{typed}}(u) \cap N_{\mathrm{typed}}(v)|$, the soft
overlap is defined for any pair of entities without a same-type
constraint; the same-type restriction $\tau_V(v) = \tau_V(u)$ is
applied when constructing the co-reference predicate
$\mathrm{Dup}_r$ from the overlap count, not in the overlap
definition itself. That is, the soft overlap compares
each attribute of~$u$ only against attributes of~$v$ that are reached
via the same relationship type --- names are compared with names,
phone numbers with phone numbers, and so on --- and sums the
resulting similarity scores.

To see concretely how the soft overlap captures evidence that exact
matching misses, consider two Person entities $u$ and $v$ in a
database where $u$ has attribute node $a_1$ = ``John Smith'' via edge
type $\tau = \textsf{hasName}$, and $v$ has attribute node $a_2$ =
``John F.\ Smith'' via the same edge type $\tau =
\textsf{hasName}$.  Since $a_1 \neq a_2$ as graph nodes, the typed
pair $(a_1, \tau)$ does not equal $(a_2, \tau)$, and therefore
$(a_1, \tau) \notin N_{\mathrm{typed}}(u) \cap
N_{\mathrm{typed}}(v)$.  In the exact predicate, this name evidence
contributes~$0$ to $|N_{\mathrm{typed}}(u) \cap
N_{\mathrm{typed}}(v)|$.  In the soft overlap, however, a
well-trained feature encoder will produce similar embeddings
$\psi(a_1) \approx \psi(a_2)$, so that
$\mathrm{sim}\bigl(\psi(a_1), \psi(a_2)\bigr)$ is close to~$1$ ---
say $0.92$ --- and the name match contributes $0.92$ to
$\mathrm{ov}_{\mathrm{soft}}(u, v)$ rather than~$0$.  If $u$ and $v$
additionally share an exact phone number --- both pointing to the
same phone attribute node $a_3$ --- then
$\mathrm{sim}\bigl(\psi(a_3), \psi(a_3)\bigr) = 1$, and the total
soft overlap becomes $0.92 + 1.0 = 1.92$.  A soft threshold at
$r = 2$ would recognise this as near-certain evidence of
co-reference, whereas the exact predicate $\mathrm{Dup}_2$ would
return~$0$ because only one attribute node is literally shared.

Now consider the noiseless limit: the data is perfectly clean, every
real-world attribute value corresponds to a unique canonical attribute
node, and the feature encoder is injective up to similarity.  In this
limit,
\[
\mathrm{sim}\bigl(\psi(a),\, \psi(a')\bigr)
\;=\; \mathbf{1}[a = a']
\quad \text{for all } a, a' \in \mathcal{A},
\]
and the soft overlap reduces to
\[
\mathrm{ov}_{\mathrm{soft}}(u, v)
\;=\; \sum_{(a,\,\tau) \in N_{\mathrm{typed}}(u)}
\mathbf{1}\bigl[(a, \tau) \in N_{\mathrm{typed}}(v)\bigr]
\;=\; \bigl|N_{\mathrm{typed}}(u) \cap N_{\mathrm{typed}}(v)\bigr|.
\]
In particular, $\mathrm{Dup}_r(u) = 1$ if and only if there exists
$v \in \mathcal{E} \setminus \{u\}$ with $\tau_V(v) = \tau_V(u)$
and $\mathrm{ov}_{\mathrm{soft}}(u, v) \geq r$, which in the
noiseless limit recovers exactly the definition from
Section~\ref{sec:prob-formulation}.  The exact predicates are not a
separate problem from practical soft matching; they are its cleanest
special case.

\medskip
\noindent\textbf{Structural context beyond pairwise attribute
	comparison.}
A second advantage of GNNs over flat attribute-comparison models is
that they capture \emph{structural context} from the graph.
Traditional entity resolution systems extract a feature vector from
each record independently, then compare records pairwise via
attribute-specific similarity functions.  This approach never examines
the broader graph: it does not know how many \emph{other} entities
share a given attribute value.

A GNN, by contrast, accesses this information through message passing.
In a typed entity-attribute graph, a two-layer MPNN with reverse
message passing computes, at each attribute node~$a$, the number of
entities of each type that connect to~$a$ via each edge type --- the
fan-in profile of the attribute.  This fan-in profile is a measure of
\emph{attribute specificity}: a phone number shared by exactly two
entities is strong evidence of co-reference, while a phone number
shared by hundreds of entities --- as occurs when a registered-agent
service provides a single contact number for many shell companies ---
is essentially noise.  After the reverse pass, entity~$u$'s
embedding~$h^{(2)}(u)$ encodes not just $u$'s own attributes but the
fan-in profile at each of its attribute neighbours, enabling the GNN
to weight shared attributes by their discriminative value.  The
sufficiency construction for Theorem~\ref{thm:K21-simple} formalises
precisely this two-layer forward-reverse mechanism.

A flat pairwise comparison model has no access to these fan-in
profiles.  It compares $u$'s attributes to $v$'s attributes without
knowing whether their shared phone number is unique to the two of
them or is shared by dozens of unrelated entities.  The graph
structure encodes this contextual information, and the MPNN extracts
it through message passing.

\medskip
\noindent\textbf{Expressivity is a necessary condition for
	learnability.}
The noiseless reduction has a direct consequence for architecture
selection.  A standard principle in representation learning is that a
model cannot learn a target function that it cannot
represent~\cite{Xu-et-al:ICLR:2019}: if no setting of the model's
parameters computes the target function correctly, then no training
procedure, no dataset, and no hyperparameter search can make the
model succeed.

Suppose $f$ is any soft co-reference scoring function on typed
entity-attribute graphs that agrees with $\mathrm{Dup}_r$ on the
subclass of graphs where all attribute values are canonical.  If an
MPNN architecture class~$\mathcal{C}$ cannot compute
$\mathrm{Dup}_r$ on all canonical-valued graphs, then $\mathcal{C}$
cannot compute~$f$ on the larger class of all graphs either, because
the canonical-valued graphs form a subclass on which $f$ and
$\mathrm{Dup}_r$ coincide by assumption.  In other words, the
impossibility results established in this paper are \emph{a fortiori}
impossibility results for any reasonable soft generalisation of
$\mathrm{Dup}_r$: the architectural requirements we identify ---
reverse message passing for $K_{2,1}$, ego IDs for $K_{2,r}$ with
$r \geq 2$ --- are necessary conditions for any GNN-based entity
resolution system whose behaviour on clean inputs is consistent
with~$\mathrm{Dup}_r$.

\medskip
\noindent\textbf{Empirical corroboration.}
Consistent with these arguments, several GNN-based
entity resolution systems have demonstrated gains on noisy data,
including GraphER~\cite{Li:GraphER:AAAI:2020},
HierGAT~\cite{Yao:HierGAT:SIGMOD:2022}, and Ganesan et
al.~\cite{Ganesan:IBM:2020}.  In the related problem of
non-obvious relationship detection in knowledge graphs,
M\"{u}ller et al.~\cite{Muller:EDBT:2020} demonstrate that
integrating graph-structural features via GNNs yields
substantial improvements over attribute-only baselines ---
notably in settings such as fake identity detection, where
related entities may share no common attributes because
attributes are deliberately disguised, making attribute
comparison alone insufficient and graph structure essential as a
complementary signal.  The typed entity-attribute graph formalism
introduced in Section~\ref{sec:prob-formulation} captures the
structural backbone over which these systems operate; the
expressivity results in Section~\ref{sec:main-results}
characterise what the underlying message-passing architecture can
and cannot compute over this structure.

\medskip
\noindent\textbf{The role of the present paper.}
Given these advantages of GNNs for entity resolution, a natural
question is: which architectural features are needed for which entity
resolution tasks?  The four theorems in this paper provide a precise
answer.  They map each entity resolution task --- shared-attribute
detection, multi-attribute co-reference, cycle detection --- and each
graph class --- simple versus multigraph --- to the minimal MPNN
architecture that can \emph{represent} the corresponding target
function.  Since representability is a prerequisite for learnability,
these results provide hard lower bounds on the architectural
requirements for any GNN-based entity resolution system, whether it
operates on exact attribute values or on learned soft embeddings.

\section{Related Work}

Graph neural networks were introduced by Gori et
al.~\cite{Gori:2005} and Scarselli et al.~\cite{Scarselli:2009} as a
framework for learning node and graph representations by iteratively
aggregating information from local neighbourhoods.  The modern
formulation as message-passing neural networks, in which each node
updates its embedding by aggregating messages from its neighbours, is
now standard; see Hamilton~\cite{Hamilton:GRL:2020} for a
comprehensive treatment.

The foundational result of Xu et al.~\cite{Xu-et-al:ICLR:2019}
establishes that the expressive power of MPNNs is bounded by the
1-Weisfeiler--Leman graph isomorphism test: any two graphs
distinguishable by an MPNN are also distinguished by 1-WL, and the
Graph Isomorphism Network achieves this upper bound.  This
characterisation reveals fundamental limitations --- for instance,
1-WL assigns all nodes the same color in vertex-transitive graphs
--- and motivates architectural extensions that exceed the 1-WL
barrier.

Two prominent mechanisms for surpassing the 1-WL bound are port
numbering and ego IDs.  Port numbering, which assigns distinguishing
integer labels to edges incident at each node, was studied by Sato et
al.~\cite{Sato-et-al:NeurIPS:2019} as a means of increasing MPNN
expressivity beyond 1-WL on simple graphs.  Ego IDs, introduced by
You et al.~\cite{You-et-al:AAAI:2021}, equip a designated center
node with a binary indicator feature, enabling the MPNN to track
information relative to a specific node.  Both mechanisms increase
expressivity at increased computational cost: ego IDs in particular
require one forward pass per node, multiplying the per-layer cost by
$O(|V|)$.

Egressy et al.~\cite{Egressy-et-al:AAAI:2024} study MPNNs on
directed multigraphs, adapting the three mechanisms above ---
reverse message passing, directed multigraph port numbering, and ego
IDs --- to the directed setting.  Their main result is a
\emph{universality} theorem: the combination of all three adaptations
enables unique node-ID assignment on connected directed multigraphs,
and hence detection of any directed subgraph
pattern~\cite[Corollary~4.4.1]{Egressy-et-al:AAAI:2024}.  However,
Egressy et al.\ establish only that the full triple is
\emph{sufficient}; they do not characterise which individual
adaptations are necessary for specific tasks, nor do they provide
lower bounds on the required depth.  The present paper fills this gap
with tight separation theorems for the entity resolution predicates
$\mathrm{Dup}_r$ and $\mathrm{Cyc}_\ell$.

GNN-based entity resolution is an active and growing area of applied
research.  Li et al.~\cite{Li:GraphER:AAAI:2020} introduce GraphER,
which constructs an entity-record graph and applies graph convolutional networks with cross-encoding and a token-gating mechanism to perform token-centric entity matching.  Yao et al.~\cite{Yao:HierGAT:SIGMOD:2022} propose HierGAT, which builds a hierarchical heterogeneous graph over entity attributes and tokens, and combines graph attention with Transformer-based contextual embeddings to derive entity similarity embeddings for entity resolution, demonstrating robustness on dirty datasets. Zhang et al.~\cite{Zhang:Autoblock:2020} propose
AutoBlock, a learning-based blocking framework that uses
similarity-preserving representation learning and locality-sensitive
hashing to generate candidate pairs for entity matching, eliminating
the need for hand-crafted blocking keys.  On the link prediction
side, Ganesan et al.~\cite{Ganesan:IBM:2020} study link prediction
with graph neural networks on property graphs in master data management,
addressing the practical challenges of anonymisation, explainability,
and production deployment under enterprise ethical and privacy constraints.  Kumar et
al.~\cite{Kumar:COMAD:2024} incorporate document structure measures
into relational graph convolutional networks for relation extraction,
showing improved accuracy on datasets relevant to semantic automation
tasks in data management systems.

Schlichtkrull et al.~\cite{Schlichtkrull:ESWC:2018} introduce relational graph convolutional networks (R-GCNs) for multi-relational data, using relation-specific weight matrices to handle typed edges --- a design choice directly relevant to our typed entity-attribute graphs, where distinct edge types encode different attribute relationships. However, R-GCNs and related architectures are evaluated purely empirically, without formal analysis of their expressive power; it remains unclear which graph structures they can distinguish and which architectural components are theoretically necessary. Our paper addresses this gap by providing the theoretical foundation explaining \emph{why} different entity resolution tasks require different architectural features, and \emph{which} features can be safely omitted for a given task.

Explainability of GNN predictions is an important concern when
deploying entity resolution systems in enterprise settings.
Chen et al.~\cite{Chen:SIGMOD:2024} propose GVEX, a view-based
explanation framework for GNN-based graph classification that
generates two-tier explanation structures consisting of graph
patterns and induced explanation subgraphs, with provable
approximation guarantees.
Zhu et al.~\cite{Zhu:SliceGX:arXiv:2025} introduce SliceGX, a
layer-wise GNN explanation method based on model slicing that
generates explanations at specific intermediate GNN layers,
enabling progressive model diagnosis; notably, SliceGX targets
node classification tasks including fraud detection scenarios
more directly analogous to entity resolution pipelines.
Both frameworks provide formal approximation guarantees for the
optimisation objectives of their respective explanation generation
problems, while the quality of the resulting explanations is
assessed empirically via fidelity metrics rather than through
formal analysis of the expressive power of the underlying GNN
architectures.
While these works address the orthogonal question of
\emph{explaining} GNN outputs rather than characterising their
\emph{expressive power}, they highlight the practical importance
of understanding GNN behaviour in entity resolution pipelines.

More broadly, Khan et
al.~\cite{Khan:SIGMOD:2025} survey the synergies between graph data
management and graph machine learning across the full data pipeline,
including graph data cleaning, scalable embedding, GNN training, and
explainability, identifying entity resolution as a key application
that benefits from the integration of both fields.

\subsection{Main Contributions}

The main contributions of this paper, contrasted with the prior work
discussed above, are as follows.

\begin{enumerate}
	\item \textbf{Typed entity-attribute graph formalism.}  We introduce
	the typed entity-attribute graph as a formal model for entity
	resolution data, together with the $K_{2,r}$ co-reference predicates
	$\mathrm{Dup}_r$ and the $\ell$-cycle participation predicate
	$\mathrm{Cyc}_\ell$.  These predicates capture the core
	computational patterns underlying duplicate detection and cyclic
	fraud detection in master data management.
	
	\item \textbf{Tight necessity and sufficiency for $K_{2,1}$.}
	We prove that reverse message passing is necessary and sufficient
	for $\mathrm{Dup}_1$ on simple typed entity-attribute graphs
	(Theorem~\ref{thm:K21-simple}), and that incoming port numbering is
	additionally necessary and sufficient on multigraph typed
	entity-attribute graphs (Theorem~\ref{thm:K21-multigraph}).  Both
	constructions achieve optimal depth~$2$.
	
	\item \textbf{The $K_{2,1}$--$K_{2,r}$ complexity gap (main result).}
	We prove that for $r \geq 2$, ego IDs are necessary for
	$\mathrm{Dup}_r$ even when reverse message passing and full port
	numbering are available, and that ego IDs together with reverse
	message passing suffice at optimal depth~$4$
	(Theorem~\ref{thm:K2r-simple}).  This establishes a strict
	increase in required architectural complexity in the transition from
	$\mathrm{Dup}_1$ to $\mathrm{Dup}_r$, driven by the need for
	cross-attribute identity correlation.
	
	\item \textbf{Ego IDs for cycle detection.}  We prove that ego IDs
	are necessary for $\ell$-cycle detection on typed directed graphs,
	even when reverse message passing and full port numbering are
	granted, and that ego IDs alone suffice on the canonical separation
	instances at optimal depth~$\ell$
	(Theorem~\ref{thm:cycle-detection}).
	
	\item \textbf{Minimal-architecture principle.}  The four theorems
	together yield a complete minimal-architecture principle
	(Table~\ref{tab:minimal-architecture}): for each entity resolution
	task, graph class, and adaptation set, we determine whether the
	adaptation set is sufficient, whether it is necessary, and what the
	optimal depth is.  Every entry is tight in the sense that both the
	adaptation set and the depth are simultaneously minimised.  This
	stands in contrast to the universality result of Egressy et
	al.~\cite{Egressy-et-al:AAAI:2024}, which proves sufficiency of
	the full triple but does not address necessity or minimality of
	individual components for specific tasks.
\end{enumerate}

\section{Preliminaries} \label{sec:preliminaries}

\subsection{Message Passing Neural Networks (MPNNs)}

In the present section, we define message passing neural networks (MPNNs) on
typed directed multigraphs, following the general framework of
\cite{Hamilton:GRL:2020} adapted to our typed setting. We then describe three
architectural enhancements -- reverse message passing, directed multigraph
port numbering, and ego IDs -- adapted to the directed multigraph setting by
\cite{Egressy-et-al:AAAI:2024}.

Given a typed directed multigraph
$G = (V, E, \mathrm{src}, \mathrm{dst}, \mathcal{T}_V, \mathcal{T}_E, \tau_V, \tau_E)$,
a message passing neural network of depth $K$ iteratively computes node
embeddings, as follows. Initially, each node $v \in V$ has embedding
$h^{(0)}(v)$, which encodes the node type $\tau_V(v)$ and any input features.
Message passing proceeds over $K$ layers. At layer $k \in \{1, \dots, K\}$,
each node $v$ computes a new embedding $h^{(k)}(v)$ based on its previous
embedding $h^{(k-1)}(v)$ and messages received from its predecessors
$N_{\mathrm{in}}(v)$:
\begin{enumerate}
	\item Each predecessor $u \in N_{\mathrm{in}}(v)$ sends a message along
	each edge $(u, v, \tau) \in E$, consisting of its embedding
	$h^{(k-1)}(u)$ together with the edge type $\tau$.
	\item Node $v$ aggregates the incoming messages:
	\[
	a_{\mathrm{in}}^{(k)}(v) = \mathrm{AGG}_{\mathrm{in}}^{(k)}
	\bigl( \{\!\!\{ (h^{(k-1)}(u),\, \tau) : (u, v, \tau) \in E
	\}\!\!\} \bigr)
	\]
	where $\{\!\!\{ \cdot \}\!\!\}$ denotes a multiset and
	$\mathrm{AGG}_{\mathrm{in}}^{(k)}$ is a permutation-invariant
	aggregation function such as sum, mean, or element-wise max. Note that
	in a standard directed MPNN, aggregation is only over \emph{incoming}
	neighbors.
	\item Node $v$ updates its embedding:
	\[
	h^{(k)}(v) = \mathrm{UPDATE}^{(k)} \bigl( h^{(k-1)}(v),\;
	a_{\mathrm{in}}^{(k)}(v) \bigr)
	\]
	where $\mathrm{UPDATE}^{(k)}$ is an update function. Note that since
	$h^{(k)}(v)$ is computed from $h^{(k-1)}(v)$, information from earlier
	layers -- including the initial embedding $h^{(0)}(v)$ and hence the
	node type $\tau_V(v)$ -- can be retained across layers.
\end{enumerate}
After $K$ layers, a node-level readout function $\mathrm{READOUT}$ --- typically a
small MLP or linear layer --- maps the final embedding to the predicted label:
\[
\hat{y}(v) = \mathrm{READOUT}\bigl( h^{(K)}(v) \bigr).
\]

\subsection{The 1-Weisfeiler--Leman Test and the Expressive Power of MPNNs}

The \emph{1-Weisfeiler--Leman test} (1-WL), also called
\emph{color refinement}, is a classical iterative algorithm for
graph isomorphism testing~\cite{Weisfeiler:1968, Shervashidze:2011}.
Given a graph $G$ with initial node labels, the test assigns each
node~$v$ an initial label $\ell_0(v)$ based on its input features.
At each subsequent iteration $t \geq 1$, the label of each node is
updated in three steps: first, the multiset of labels of $v$'s
neighbours is aggregated; second, this multiset is concatenated with
$v$'s own label from the previous iteration; and third, an injective
\emph{label compression} function maps the resulting pair to a new
label.  Formally:
\[
\ell_t(v) \;=\; f\!\Bigl(\ell_{t-1}(v),\;\;
\bigl\{\!\!\bigl\{ \ell_{t-1}(w) : w \in N(v) \bigr\}\!\!\bigr\}
\Bigr),
\]
where $f$ is an injective function that assigns a unique new label to
each distinct pair of a label and a multiset of labels.  The process
terminates when the partition of nodes induced by the labelling
stabilises.  Two graphs are declared \emph{non-isomorphic by 1-WL}
if their multisets of node labels differ at some iteration.

The connection to MPNNs is direct: the MPNN update rule
\[
h^{(k)}(v) = \mathrm{UPDATE}^{(k)}\!\bigl(h^{(k-1)}(v),\,
a_{\mathrm{in}}^{(k)}(v)\bigr)
\]
is a continuous, learned analogue of
the discrete 1-WL label update.  This parallel was made precise by
Xu et al.~\cite{Xu-et-al:ICLR:2019}, who proved the following upper
bound on MPNN expressivity: if a GNN maps two graphs to different
embeddings, then the 1-WL test also declares them non-isomorphic.
Consequently, any aggregation-based GNN is at most as powerful as
1-WL in distinguishing graphs and, by extension, node
neighborhoods.

Xu et al.~\cite{Xu-et-al:ICLR:2019} also showed that this upper
bound is achievable.  The \emph{Graph Isomorphism Network} (GIN) is
an MPNN that uses sum aggregation and an injective update function
implemented as a multilayer perceptron (MLP).  The key property of
GIN is that its aggregation and update are both injective, making it
as powerful as 1-WL: two nodes receive the same GIN embedding after
$K$ layers if and only if they receive the same 1-WL label after $K$
iterations.  The injectivity of sum aggregation rests on a result
that we invoke repeatedly in our sufficiency proofs: Xu et
al.~\cite[Lemma~5]{Xu-et-al:ICLR:2019} show that for any countable
feature space, there exists a function into~$\mathbb{R}^n$ such that
summing over any bounded-size multiset yields an injective mapping.
This guarantees that sum aggregation can faithfully encode entire
multisets of neighbour messages, which in our constructions enables
attribute nodes to recover the full multiset of incoming entity-type
and edge-type signatures, and entity nodes to recover the full
multiset of attribute embeddings received via reverse message passing.

\begin{Remark}[Bounded-cardinality assumption]\label{rem:bounded-multiset}
	Throughout this paper, we work in the standard countable-feature
	model of Xu et al.~\cite{Xu-et-al:ICLR:2019}: the type sets
	$\mathcal{T}_\mathcal{E}$, $\mathcal{T}_\mathcal{A}$, and
	$\mathcal{T}_E$ are finite, and the multisets encountered in each
	aggregation step have cardinality bounded by the maximum node degree
	of the input graph.  Under these assumptions --- which hold for all
	finite typed entity-attribute graphs --- the result of Xu et
	al.~\cite[Lemma~5]{Xu-et-al:ICLR:2019} guarantees the existence of
	injective sum-based aggregation functions at each layer.  All
	sufficiency constructions in this paper (Theorems~\ref{thm:K21-simple},
	\ref{thm:K21-multigraph}, and~\ref{thm:K2r-simple}) operate under
	these standard assumptions.
\end{Remark}

Although GIN achieves the 1-WL upper bound, this bound itself has
well-known limitations.  The 1-WL test assigns the same label to all
nodes in a vertex-transitive graph when all nodes start with the
same initial label; in particular, it cannot distinguish the two
disjoint copies of a directed $\ell$-cycle from a single
$2\ell$-cycle.  For the entity resolution tasks studied in this
paper, the 1-WL barrier manifests concretely: a standard MPNN
without adaptations cannot compute even the simplest co-reference
predicate $\mathrm{Dup}_1$, because entity nodes in a typed
entity-attribute graph have $N_{\mathrm{in}}(u) = \emptyset$ and
therefore receive no graph-structural information via standard
forward aggregation.  The three adaptations described next ---
reverse message passing, port numbering, and ego IDs --- are
designed to overcome these limitations.

Indeed, Egressy et al.~\cite{Egressy-et-al:AAAI:2024} prove that
the combination of all three adaptations enables the assignment of
unique node IDs in any connected directed multigraph; by a
universality result of Loukas~\cite{Loukas:ICLR:2020}, an MPNN with
unique node IDs and sufficient depth and width can compute any
Turing-computable function on graphs, and in particular can detect
any directed subgraph pattern.  While this universal architecture guarantees correctness for any
subgraph detection task, its power comes at significant computational
cost --- ego IDs alone require one forward pass per node, multiplying
the per-layer complexity by~$O(|V|)$, and full bidirectional port
numbering adds a pre-computation step of~$O(|E| \log |E|)$.
Since an architecture that provably cannot compute a target function
will fail to learn it regardless of training procedure, dataset size,
or overparameterisation, the natural question is: which subsets of
these adaptations are genuinely necessary for a given task, and which
can be safely omitted to reduce cost without sacrificing correctness?

\subsection{Adaptations}

We now discuss three adaptations to the standard MPNN architecture 
defined above. In this work, we investigate the expressive power of 
MPNNs when enhanced with various subsets of these adaptations.

\textbf{Adaptation 1: Reverse message passing.} We say that an MPNN is
enhanced with reverse message passing if, in each layer of message passing,
each node also receives information from each of its successors -- i.e., node
embeddings and edge types are sent in the reverse direction of each edge, and this
information is aggregated separately. More specifically, at layer $k$, each
node $v$ computes two aggregations:
\begin{align*}
	a_{\mathrm{in}}^{(k)}(v) &= \mathrm{AGG}_{\mathrm{in}}^{(k)}
	\bigl( \{\!\!\{ (h^{(k-1)}(w),\, \tau) : (w, v, \tau) \in E
	\}\!\!\} \bigr), \\
	a_{\mathrm{out}}^{(k)}(v) &= \mathrm{AGG}_{\mathrm{out}}^{(k)}
	\bigl( \{\!\!\{ (h^{(k-1)}(w),\, \tau) : (v, w, \tau) \in E
	\}\!\!\} \bigr),
\end{align*}
and the node embedding is updated using both:
\[
h^{(k)}(v) = \mathrm{UPDATE}^{(k)} \bigl( h^{(k-1)}(v),\;
a_{\mathrm{in}}^{(k)}(v),\; a_{\mathrm{out}}^{(k)}(v) \bigr).
\]
This enables information to flow against the edge direction, which is
essential in typed entity-attribute graphs where all edges point from entities
to attributes: without reverse message passing, entity nodes would never
receive information from their attribute neighbors. This adaptation was
introduced by Jaume et al.~\cite{Jaume:RLGM:ICLR:2019} for directed labeled
graphs and subsequently adopted by Egressy et
al.~\cite{Egressy-et-al:AAAI:2024} as one of the three adaptations for
directed multigraphs.

\textbf{Adaptation 2: Directed multigraph port numbering.}
In a directed multigraph, a node $v$ may have multiple incoming edges from
distinct predecessors, as well as parallel edges from the same predecessor.
Port numbering assigns integer labels to edges that allow nodes to
distinguish these cases.

Formally, each edge $e \in E$ is assigned an \emph{incoming port number}
$p_{\mathrm{in}}(e) \in \mathbb{N}$ and an \emph{outgoing port number}
$p_{\mathrm{out}}(e) \in \mathbb{N}$ satisfying the following properties:
\begin{enumerate}
	\item[(i)] All parallel edges from the same source $u$ to the same
	destination $v$ share the same incoming port at $v$ and the same
	outgoing port at $u$.
	\item[(ii)] Edges from \emph{distinct} sources into the same
	destination receive \emph{distinct} incoming ports.
	\item[(iii)] Edges to \emph{distinct} destinations from the same
	source receive \emph{distinct} outgoing ports.
\end{enumerate}

Both port numbers $(p_{\mathrm{in}}(e),\, p_{\mathrm{out}}(e))$ are
attached as edge features and are visible to both endpoints.
When port numbering is combined with reverse message passing
(Adaptation~1), the aggregations at layer $k$ become:
\begin{align*}
	a_{\mathrm{in}}^{(k)}(v) &= \mathrm{AGG}_{\mathrm{in}}^{(k)}
	\bigl( \{\!\!\{ (h^{(k-1)}(w),\, \tau,\,
	p_{\mathrm{in}}(e),\, p_{\mathrm{out}}(e))
	: e = (w, v, \tau) \in E
	\}\!\!\} \bigr), \\
	a_{\mathrm{out}}^{(k)}(v) &= \mathrm{AGG}_{\mathrm{out}}^{(k)}
	\bigl( \{\!\!\{ (h^{(k-1)}(w),\, \tau,\,
	p_{\mathrm{in}}(e),\, p_{\mathrm{out}}(e))
	: e = (v, w, \tau) \in E
	\}\!\!\} \bigr),
\end{align*}
where each message now includes the port numbers of the edge that
carries it.  Without reverse message passing, only
$a_{\mathrm{in}}^{(k)}(v)$ is computed.

This convention follows \cite{Egressy-et-al:AAAI:2024}, which adapts port
numbering from \cite{Sato-et-al:NeurIPS:2019} to directed multigraphs. Intuitively, the incoming port number at a node $v$ amounts to an
arbitrary ordering of $v$'s distinct predecessors, and the outgoing port
number at a node $u$ amounts to an arbitrary ordering of $u$'s distinct
successors. Parallel edges between the same pair of nodes share the same
port numbers.

We distinguish two sub-adaptations that may be applied independently.
\emph{Incoming port numbering} (Adaptation~2a) assigns only
$p_{\mathrm{in}}(e)$ satisfying properties (i)--(ii) above.
\emph{Outgoing port numbering} (Adaptation~2b) assigns only
$p_{\mathrm{out}}(e)$ satisfying properties (i) and (iii) above. The full
scheme (Adaptation~2) uses both.  Individual theorems below specify precisely which components are required.

An important subtlety is that port numbering is an \emph{adversarial}
resource: the MPNN has no control over which specific port numbering is
realized. Accordingly, when we say that port numbering is
\emph{sufficient} to compute a ground-truth predicate $f(v)$, we mean that the MPNN
computes $f(v)$ correctly under \emph{every} valid port numbering
assignment. When we say that port numbering is \emph{not sufficient}, we
exhibit specific graphs and a valid port numbering under which no MPNN can
distinguish nodes with different $f$-values.

The concept of port numbering as a symmetry-breaking resource for
anonymous networks originates in the theory of distributed local
algorithms, where each node must produce its output after a constant
number of synchronous communication rounds with its
neighbours~\cite{Angluin:STOC:1980, Linial:SIAMJC:1992,
	Naor-Stockmeyer:SIAMJC:1995}.  Sato et
al.~\cite{Sato-et-al:NeurIPS:2019} bridge this classical theory to
the GNN setting by proving that each class of GNNs they consider is
computationally equivalent to a corresponding class of distributed
local algorithms.  Working on simple undirected graphs, they show that
GNNs equipped with consistent port numbering (their CPNGNNs) are
strictly more expressive than GNNs without it, including
GIN~\cite{Xu-et-al:ICLR:2019}.  This equivalence allows impossibility
and possibility results from the distributed computing literature to
transfer directly to GNNs.  Egressy et
al.~\cite{Egressy-et-al:AAAI:2024} later adapt this port numbering
mechanism to directed multigraphs, as what we term Adaptation~2 in
the present work.

\textbf{Adaptation 3: Ego IDs.} For a designated center node $u \in V$,
define the binary feature $\mathrm{ego}_u(v) = \mathbf{1}[v = u]$, where
$\mathbf{1}[\cdot]$ denotes the indicator function, which
equals 1 if $v$ is the designated center node and 0 otherwise. This
feature is appended to the initial embedding:
\[
h^{(0)}(v) = \bigl(\tau_V(v),\; \mathrm{ego}_u(v)\bigr).
\]
To classify node $u$, the MPNN is run with $u$ as the ego node. This
breaks the symmetry between $u$ and all other nodes of the same type,
enabling the MPNN to track information relative to a specific node.
The ego ID mechanism was introduced by \cite{You-et-al:AAAI:2021} and
applied to directed multigraphs in \cite{Egressy-et-al:AAAI:2024}. 

Note that without ego IDs, the per-layer cost of an MPNN is
dominated by the graph size: the MPNN computes one message per edge
and one aggregation-plus-update per node, yielding $O(|E| + |V|)$
operations per layer when the embedding dimension is treated as a
constant independent of graph size.  Since a typed entity-attribute
graph has no isolated nodes --- every entity has at least one
outgoing edge, and every attribute has at least one incoming edge
--- we have $|V| \leq |E| + 1$, and the per-layer cost simplifies
to $O(|E|)$.  With ego IDs, classifying every node requires running
the MPNN once per node, increasing the total per-layer cost from $O(|E|)$ to
$O(|V| \cdot |E|)$.  As we shall see in
Theorem~\ref{thm:K2r-simple}, this $O(|V|)$-factor overhead is not
an artifact of a particular architecture but an inherent cost: ego
IDs are necessary for detecting $K_{2,r}$ patterns with $r \geq 2$,
and no MPNN without ego IDs can compute $\mathrm{Dup}_r$ on all
simple typed entity-attribute graphs, regardless of depth or width.

\section{Main Results} \label{sec:main-results}

\subsection{$K_{2,1}$ Detection on Simple Entity-Attribute Graphs}

We now study which MPNN adaptations are necessary and sufficient for
computing the duplicate predicate $\mathrm{Dup}(u)$ on simple typed
entity-attribute graphs. Recall that $\mathrm{Dup}(u) = 1$ if and only
if there exists another entity $v$ of the same type as $u$ such that
$N_{\mathrm{typed}}(u) \cap N_{\mathrm{typed}}(v) \neq \emptyset$.

\begin{Theorem} \label{thm:K21-simple}
	On simple typed entity-attribute graphs, reverse message passing is
	necessary and sufficient for computing $\mathrm{Dup}(u)$ in the MPNN
	framework. That is:
	\begin{enumerate}
		\item[(a)] \textbf{(Necessity.)} Every MPNN without reverse message
		passing fails to compute $\mathrm{Dup}(u)$ on some simple typed
		entity-attribute graph.
		\item[(b)] \textbf{(Sufficiency.)} There exists a $2$-layer MPNN with
		reverse message passing and sum aggregation that computes
		$\mathrm{Dup}(u)$ for all entities $u \in \mathcal{E}$ in any simple
		typed entity-attribute graph.
	\end{enumerate}
\end{Theorem}

\begin{proof}[Proof of part (a): necessity]
	We show that any directed MPNN \emph{without} reverse message passing
	fails to compute $\mathrm{Dup}$.
	
	In a typed entity-attribute graph, every edge is directed from
	$\mathcal{E}$ to $\mathcal{A}$, so every entity node
	$u \in \mathcal{E}$ satisfies $N_{\mathrm{in}}(u) = \emptyset$. A
	directed MPNN without reverse message passing aggregates only over
	incoming neighbors. Therefore, at every layer $k \geq 1$:
	\[
	a_{\mathrm{in}}^{(k)}(u)
	= \mathrm{AGG}_{\mathrm{in}}^{(k)}\bigl(\emptyset\bigr)
	= c^{(k)},
	\]
	which is a fixed constant depending only on the architecture, not on the
	graph. Consequently,
	\[
	h^{(k)}(u) = \mathrm{UPDATE}^{(k)}\!\bigl(h^{(k-1)}(u),\; c^{(k)}\bigr),
	\]
	and by induction on $k$, the embedding $h^{(K)}(u)$ depends only on the
	initial embedding $h^{(0)}(u)$ and the architectural constants
	$c^{(1)}, \ldots, c^{(K)}$. It is independent of the edge structure of
	$G$.
	
	Now consider two simple typed entity-attribute graphs that contain the same
	entity $u$ with the same initial features $h^{(0)}(u)$, and with $G_1$ and
	$G_2$ defined as follows; see Figure~\ref{fig:k21-separation}.
	In $G_1$, the \emph{positive instance}, there are two entities $u$ and $v$,
	both of type $\sigma \in \mathcal{T}_\mathcal{E}$, and a single attribute
	node $a$ of type $\alpha \in \mathcal{T}_\mathcal{A}$, connected by directed
	edges $(u, a, \tau_1)$ and $(v, a, \tau_1)$; thus
	$\mathrm{Dup}_{G_1}(u) = 1$.
	In $G_2$, the \emph{negative instance}, there is only entity $u$ of type
	$\sigma$ and attribute $a$ of type $\alpha$, connected by the single edge
	$(u, a, \tau_1)$; thus $\mathrm{Dup}_{G_2}(u) = 0$.
	
	\begin{figure}[htbp]
		\centering
		\begin{tikzpicture}[
			node/.style  = {circle, draw, thick,
				minimum size=9mm, inner sep=1pt, font=\small},
			tlbl/.style  = {font=\scriptsize, text=gray!75!black},
			>={Stealth[length=6pt, width=4pt]},
			semithick
			]
			
			\node[node, label={[tlbl]left:$\sigma$}]  (u1) at (0.0,  1.5) {$u$};
			\node[node, label={[tlbl]left:$\sigma$}]  (v1) at (0.0, -1.5) {$v$};
			\node[node, label={[tlbl]right:$\alpha$}] (a1) at (3.0,  0.0) {$a$};
			
			\draw[->] (u1) to node[above, sloped, font=\scriptsize]{$\tau_1$} (a1);
			\draw[->] (v1) to node[below, sloped, font=\scriptsize]{$\tau_1$} (a1);
			
			\node[font=\small\itshape] at (1.5, -2.6)
			{$G_1$:\enspace$\mathrm{Dup}(u)=1$};
			
			\node[node, label={[tlbl]left:$\sigma$}]  (u2) at (6.2,  0.0) {$u$};
			\node[node, label={[tlbl]right:$\alpha$}] (a2) at (9.0,  0.0) {$a$};
			
			\draw[->] (u2) to node[above, font=\scriptsize]{$\tau_1$} (a2);
			
			\node[font=\small\itshape] at (7.6, -2.6)
			{$G_2$:\enspace$\mathrm{Dup}(u)=0$};
			
		\end{tikzpicture}
		\caption{Separation graphs for the necessity proof of
			Theorem~\ref{thm:K21-simple}.}
		\label{fig:k21-separation}
	\end{figure}
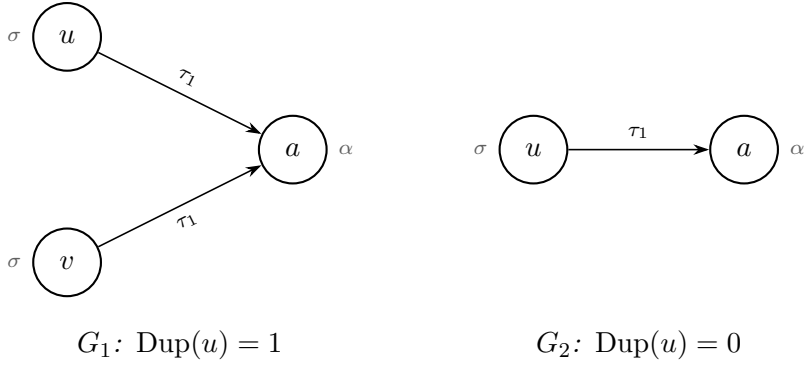
	
	Both graphs are simple. Since $h^{(K)}(u)$ is the same in both graphs --- depending only on $h^{(0)}(u)$ and the architectural constants --- every
	readout function computes the same value $\hat{y}(u)$ on both, and hence
	cannot distinguish $\mathrm{Dup}_{G_1}(u) = 1$ from
	$\mathrm{Dup}_{G_2}(u) = 0$.
\end{proof}

\begin{proof}[Proof of part (b): sufficiency]
	We construct an explicit $2$-layer MPNN with reverse message passing and
	sum aggregation that computes $\mathrm{Dup}(u)$ for all
	$u \in \mathcal{E}$ in any simple typed entity-attribute graph.
	
	\medskip
	\noindent\textsc{Layer 0 (Initialization).}
	Set $h^{(0)}(v) = \tau_V(v)$ for all $v \in V$.
	
	\medskip
	\noindent\textsc{Layer 1 (Forward: entity $\to$ attribute).}
	For each edge $(u, a, \tau) \in E$, define the message:
	\[
	\mathrm{msg}^{(1)}(u, a, \tau)
	= \bigl(\tau_V(u),\; \tau\bigr).
	\]
	Each attribute $a$ aggregates the multiset of incoming messages:
	\[
	\mathcal{M}(a) = \bigl\{\!\!\bigl\{
	\mathrm{msg}^{(1)}(u, a, \tau) : (u, a, \tau) \in E
	\bigr\}\!\!\bigr\}.
	\]
	By \cite[Lemma 5]{Xu-et-al:ICLR:2019}, a sum-based aggregation with an
	appropriate function maps $\mathcal{M}(a)$ injectively into a
	fixed-dimensional space. Let $h^{(1)}(a)$ encode this injective image.
	Since $\mathcal{M}(a)$ determines the multiset of (entity-type,
	edge-type) pairs incident to $a$, an MLP applied to $h^{(1)}(a)$ can
	compute, for each entity type $\sigma \in \mathcal{T}_\mathcal{E}$ and
	edge type $\tau_i \in \mathcal{T}_E$, the multiplicity:
	\[
	\mathrm{mult}_{\mathcal{M}(a)}(\sigma, \tau_i)
	= \bigl|\bigl\{
	s \in \mathcal{M}(a) : s = (\sigma, \tau_i)
	\bigr\}\bigr|.
	\]
	
	\medskip
	\noindent\textsc{Layer 2 (Reverse: attribute $\to$ entity).}
	For each edge $(u, a, \tau) \in E$, define the reverse message:
	\[
	\mathrm{msg}^{(2)}(u, a, \tau)
	= \bigl(h^{(1)}(a),\; \tau\bigr).
	\]
	Each entity $u$ aggregates the multiset of reverse messages:
	\[
	\mathcal{M}^{(2)}(u) = \bigl\{\!\!\bigl\{
	\mathrm{msg}^{(2)}(u, a, \tau) : (u, a, \tau) \in E
	\bigr\}\!\!\bigr\}.
	\]
	Each element of $\mathcal{M}^{(2)}(u)$ is a pair $(S, \tau')$, where $S$
	is the injectively encoded multiset $\mathcal{M}(a)$ for the attribute $a$
	at the other end of the edge, and $\tau'$ is the edge type. Again by \cite[Lemma 5]{Xu-et-al:ICLR:2019}, $h^{(2)}(u)$ injectively encodes
	$\mathcal{M}^{(2)}(u)$.
	
	\medskip
	\noindent\textsc{Readout.}
	From $h^{(2)}(u)$, an MLP computes:
	\[
	\hat{y}(u) = \mathbf{1}\Bigl[\,\exists\; (S, \tau') \in
	\mathcal{M}^{(2)}(u) \;\text{such that}\;
	\mathrm{mult}_S\bigl(\tau_V(u),\, \tau'\bigr) \geq 2 \,\Bigr].
	\]
	
	\medskip
	\noindent\textsc{Correctness.} We show
	$\hat{y}(u) = \mathrm{Dup}(u)$ for all entities $u \in \mathcal{E}$.
	
	\medskip
	\noindent$(\Rightarrow)$ Suppose $\mathrm{Dup}(u) = 1$. Then there exist
	$v \neq u$ with $\tau_V(v) = \tau_V(u)$, an attribute
	$a \in \mathcal{A}$, and an edge type $\tau_i \in \mathcal{T}_E$ such
	that both edges $(u, a, \tau_i)$ and $(v, a, \tau_i)$ exist in $E$. The
	multiset $\mathcal{M}(a)$ therefore contains at least two copies of
	$(\tau_V(u), \tau_i)$, so
	$\mathrm{mult}_{\mathcal{M}(a)}(\tau_V(u), \tau_i) \geq 2$. Since the
	edge $(u, a, \tau_i)$ exists, the pair
	$(S, \tau') = (\mathcal{M}(a), \tau_i)$ belongs to
	$\mathcal{M}^{(2)}(u)$, and the readout condition is satisfied. Hence
	$\hat{y}(u) = 1$.
	
	\medskip
	\noindent$(\Leftarrow)$ Suppose $\hat{y}(u) = 1$. Then there exists
	$(S, \tau') \in \mathcal{M}^{(2)}(u)$ with
	$\mathrm{mult}_S(\tau_V(u), \tau') \geq 2$. This pair arises from some
	edge $(u, a, \tau') \in E$, with $S = \mathcal{M}(a)$. The condition
	$\mathrm{mult}_{\mathcal{M}(a)}(\tau_V(u), \tau') \geq 2$ means at
	least two edges of type $\tau'$ from entities of type $\tau_V(u)$
	terminate at $a$. Since the graph is \emph{simple}, there is at most one
	edge of type $\tau'$ from any given entity to $a$. Therefore the two or
	more edges must come from \emph{distinct} entities. One is $u$ itself;
	the other is some $v \neq u$ with $\tau_V(v) = \tau_V(u)$. Hence
	$(a, \tau') \in N_{\mathrm{typed}}(u) \cap N_{\mathrm{typed}}(v) \neq
	\emptyset$, so $\mathrm{Dup}(u) = 1$.
	
	\medskip
	Therefore $\hat{y}(u) = \mathrm{Dup}(u)$ for all entities
	$u \in \mathcal{E}$.
\end{proof}

\begin{Remark}[Depth minimality]\label{rem:depth-K21}
	The $2$-layer depth in part~(b) is optimal: no $1$-layer MPNN with
	reverse message passing can compute $\mathrm{Dup}(u)$, even on simple
	typed entity-attribute graphs.  After one layer of reverse message
	passing, entity~$u$ aggregates the multiset
	$\{\!\!\{(h^{(0)}(a),\,\tau) : (u,a,\tau) \in E\}\!\!\}
	= \{\!\!\{(\tau_V(a),\,\tau) : (u,a,\tau) \in E\}\!\!\}$,
	since the reverse messages carry the layer-$0$ attribute embeddings.
	The resulting embedding~$h^{(1)}(u)$ therefore encodes~$\tau_V(u)$
	together with the multiset of (attribute-type, edge-type) pairs
	incident to~$u$---that is, $u$'s own typed neighborhood
	$N_{\mathrm{typed}}(u)$---but contains no information about which
	other entities share those attribute neighbors.  In particular, the
	separation graphs $G_1$ and~$G_2$ from part~(a) satisfy
	$N_{\mathrm{typed}}(u) = \{(a,\tau_1)\}$ in both graphs, so
	$h^{(1)}(u)$ is identical in both, yet
	$\mathrm{Dup}_{G_1}(u) = 1 \neq 0 = \mathrm{Dup}_{G_2}(u)$.
	
Intuitively, computing $\mathrm{Dup}(u)$ requires $2$-hop
information along the bipartite structure: in the first layer,
entity embeddings propagate forward to attribute neighbors,
allowing each attribute to learn the types of its incident
entities; in the second layer, this aggregated fan-in
information flows back to entities via reverse message passing,
enabling each entity to detect whether a same-type competitor
shares any of its attributes.  Neither hop can be eliminated.
\end{Remark}

\subsection{$K_{2,1}$ Detection on Multigraph Entity-Attribute Graphs}

Theorem~\ref{thm:K21-simple} establishes that reverse message passing
alone is necessary and sufficient for computing $\mathrm{Dup}$ on
\emph{simple} typed entity-attribute graphs.  We now consider the
multigraph setting, where parallel edges of the same type between the
same entity--attribute pair are permitted.  Such parallel edges arise
in MDM pipelines that ingest records from multiple sources or time
periods without deduplication: if a Patient entity visits the same
Provider three times, and the claim records are loaded into the
entity-attribute graph without consolidation, three parallel
$\textsf{treatedBy}$ edges result.  Whether the timestamps are
retained as edge features or discarded is a system design choice; in
either case, the underlying graph is a multigraph, and the MPNN must
handle parallel edges correctly.  Directed multigraphs are a common data model in financial transaction
networks~\cite{Gounoue-et-al:TeMP-TraG:2025,
	Egressy-et-al:AAAI:2024}, where Gounoue et al.\ report an average of
$17.92$~edges per node pair in the IBM anti-money-laundering dataset.

In a multigraph, an attribute node receiving two messages with the same
(entity-type, edge-type) signature cannot determine---without additional
information---whether those messages originated from two \emph{distinct}
entities (a genuine shared-attribute signal) or from two \emph{parallel
	edges} of the same entity (a recording artifact). As the next theorem
shows, this is precisely the ambiguity that incoming port numbering
(Adaptation~2a) resolves.

\begin{Theorem}\label{thm:K21-multigraph}
	On typed entity-attribute graphs in which parallel edges are permitted, i.e. not necessarily simple:
	\begin{enumerate}
		\item[(a)] \textbf{(Necessity of port numbering.)} There exist such
		graphs on which every MPNN with reverse message passing but without
		port numbering fails to compute $\mathrm{Dup}(u)$.
		\item[(b)] \textbf{(Sufficiency of incoming port numbering.)} There
		exists a $2$-layer MPNN with reverse message passing, incoming port
		numbering (Adaptation~2a), and sum aggregation that computes
		$\mathrm{Dup}(u)$ for all entities $u \in \mathcal{E}$ in any such
		graph.
	\end{enumerate}
\end{Theorem}

\begin{proof}[Proof of part (a): necessity]
	We construct two multigraph typed entity-attribute graphs that are
	indistinguishable by any MPNN with reverse message passing but without
	port numbering, yet differ in their $\mathrm{Dup}$ labels; see Figure~\ref{fig:k21-multigraph-separation}.
	
	Fix an entity type $\sigma \in \mathcal{T}_\mathcal{E}$, an attribute
	type $\alpha_1 \in \mathcal{T}_\mathcal{A}$, and an edge type
	$\tau_1 \in \mathcal{T}_E$. All entity nodes are initialized with the
	same embedding $h^{(0)} = \sigma$ and all attribute nodes with
	$h^{(0)} = \alpha_1$.
	
	In $G_1$, the \emph{positive instance}, the entity set is
	$\{u, v\}$ with $\tau_V(u) = \tau_V(v) = \sigma$; the attribute
	set is $\{a_1, a_2\}$ with $\tau_V(a_1) = \tau_V(a_2) = \alpha_1$;
	and the edge set consists of four distinct edges
	$(u, a_1, \tau_1)$, $(v, a_1, \tau_1)$, $(u, a_2, \tau_1)$, and
	$(v, a_2, \tau_1)$.
	Each entity has out-degree $2$ and each attribute has in-degree $2$.
	Since $(a_1, \tau_1) \in N_{\mathrm{typed}}(u) \cap N_{\mathrm{typed}}(v)$,
	we have $\mathrm{Dup}_{G_1}(u) = 1$.
	
	In $G_2$, the \emph{negative instance}, the entity and attribute sets
	are the same, but the edges consist of two parallel edges from $u$ to
	$a_1$ of type $\tau_1$, and two parallel edges from $v$ to $a_2$ of type
	$\tau_1$. Each entity again has out-degree $2$ (counting multiplicity)
	and each attribute has in-degree $2$. Entity $u$ connects only to $a_1$
	and $v$ connects only to $a_2$, so $N_{\mathrm{typed}}(u) \cap
	N_{\mathrm{typed}}(v) = \emptyset$ and $\mathrm{Dup}_{G_2}(u) = 0$.
	
	\begin{figure}[htbp]
		\centering
		\begin{tikzpicture}[
			node/.style  = {circle, draw, thick,
				minimum size=9mm, inner sep=1pt, font=\small},
			tlbl/.style  = {font=\scriptsize, text=gray!75!black},
			>={Stealth[length=6pt, width=4pt]},
			semithick
			]
			
			\node[node, label={[tlbl]left:$\sigma$}]    (u1) at (0.0,  1.4) {$u$};
			\node[node, label={[tlbl]left:$\sigma$}]    (v1) at (0.0, -1.4) {$v$};
			\node[node, label={[tlbl]right:$\alpha_1$}] (A1) at (3.5,  0.7) {$a_1$};
			\node[node, label={[tlbl]right:$\alpha_1$}] (A2) at (3.5, -0.7) {$a_2$};
			
			\draw[->] (u1) to node[above, sloped, font=\scriptsize]{$\tau_1$} (A1);
			\draw[->] (v1) to node[below, sloped, font=\scriptsize]{$\tau_1$} (A1);
			\draw[->] (u1) to node[above, sloped, font=\scriptsize]{$\tau_1$} (A2);
			\draw[->] (v1) to node[below, sloped, font=\scriptsize]{$\tau_1$} (A2);
			
			\node[font=\small\itshape] at (1.75, -2.5)
			{$G_1$:\enspace$\mathrm{Dup}(u)=1$};
			
			\node[node, label={[tlbl]left:$\sigma$}]    (u2) at (6.5,  1.4) {$u$};
			\node[node, label={[tlbl]left:$\sigma$}]    (v2) at (6.5, -1.4) {$v$};
			\node[node, label={[tlbl]right:$\alpha_1$}] (B1) at (10.0,  0.7) {$a_1$};
			\node[node, label={[tlbl]right:$\alpha_1$}] (B2) at (10.0, -0.7) {$a_2$};
			
			\draw[->] (u2) to[bend left=12]
			node[above, sloped, font=\scriptsize]{$\tau_1$} (B1);
			\draw[->] (u2) to[bend right=12] (B1);
			
			\draw[->] (v2) to[bend left=12]
			node[above, sloped, font=\scriptsize]{$\tau_1$} (B2);
			\draw[->] (v2) to[bend right=12] (B2);
			
			\node[font=\small\itshape] at (8.25, -2.5)
			{$G_2$:\enspace$\mathrm{Dup}(u)=0$};
			
		\end{tikzpicture}
		\caption{Separation graphs for the necessity proof of
			Theorem~\ref{thm:K21-multigraph}.}
		\label{fig:k21-multigraph-separation}
	\end{figure}
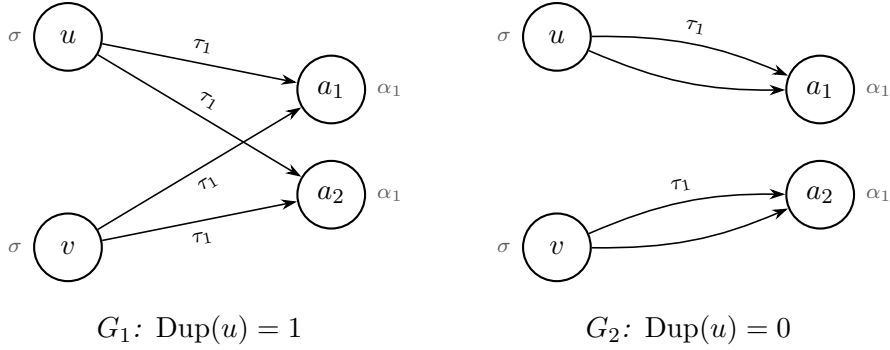
	
	We prove by induction on $k$ that at every depth $k \geq 0$, all entity
	embeddings equal a common value $\bar{h}_{\mathrm{ent}}^{(k)}$ and all
	attribute embeddings equal a common value $\bar{h}_a^{(k)}$,
	simultaneously in $G_1$ and $G_2$.
	
	\textit{Base case} ($k=0$): all entity embeddings equal $\sigma$ and all
	attribute embeddings equal $\alpha_1$ in both graphs.
	
	\textit{Inductive step}: Suppose all entity embeddings equal
	$\bar{h}_{\mathrm{ent}}$ and all attribute embeddings equal $\bar{h}_a$
	at depth $k-1$ in both graphs.
	
	\textit{Forward aggregation at each attribute}: In $G_1$, attribute $a_i$
	receives one message from $u$ and one from $v$, each of the form
	$(\bar{h}_{\mathrm{ent}}, \tau_1)$, yielding the incoming message
	multiset $\bigl\{\!\!\bigl\{(\bar{h}_{\mathrm{ent}}, \tau_1),\,
	(\bar{h}_{\mathrm{ent}}, \tau_1)\bigr\}\!\!\bigr\}$.
	In $G_2$, attribute $a_1$ receives two messages from two parallel edges
	from $u$, each also of the form $(\bar{h}_{\mathrm{ent}}, \tau_1)$.
	Without port numbering, these two multisets are identical, so
	$h^{(k)}(a_i)$ is the same in both graphs (and the same for $a_1$ and
	$a_2$ by symmetry).
	
	\textit{Reverse aggregation at entity $u$}: In $G_1$, entity $u$ has one
	edge to $a_1$ and one to $a_2$, receiving the reverse-message multiset
	$\bigl\{\!\!\bigl\{(\bar{h}_a, \tau_1),\, (\bar{h}_a,
	\tau_1)\bigr\}\!\!\bigr\}$.
	In $G_2$, entity $u$ has two parallel edges to $a_1$, receiving the
	reverse-message multiset $\bigl\{\!\!\bigl\{(\bar{h}_a, \tau_1),\,
	(\bar{h}_a, \tau_1)\bigr\}\!\!\bigr\}$.
	These are identical, so $h^{(k)}(u)$ is the same in both graphs.
	
	By induction, $h^{(K)}(u)$ is the same in $G_1$ and $G_2$ for every
	depth $K$, and so the predicted label $\hat{y}(u)$ takes the same value in both graphs. Since $\mathrm{Dup}_{G_1}(u) = 1 \neq 0 =
	\mathrm{Dup}_{G_2}(u)$, no readout function can distinguish the two
	labels.
\end{proof}

\begin{proof}[Proof of part (b): sufficiency]
	We construct an explicit $2$-layer MPNN with reverse message passing,
	incoming port numbering (Adaptation~2a), and sum aggregation that
	computes $\mathrm{Dup}(u)$ for all $u \in \mathcal{E}$ in any multigraph
	typed entity-attribute graph.
	
	\medskip
	\noindent\textsc{The key challenge.}
	In a simple typed entity-attribute graph, the multiplicity
	$\mathrm{mult}_{\mathcal{M}(a)}(\sigma, \tau_i) \geq 2$ at an attribute
	$a$ guarantees at least two distinct entities of type $\sigma$ connected
	to $a$ via edge type $\tau_i$, because each entity contributes at most
	one such edge. In a multigraph, this fails: a single entity may have
	multiple parallel edges of type $\tau_i$ to the same $a$, inflating the
	raw multiplicity without implying a second entity. Incoming port
	numbering resolves this. By Adaptation~2a, all parallel edges from the
	same source to the same destination share the same incoming port number,
	while edges from \emph{distinct} sources to the same destination receive
	\emph{distinct} incoming port numbers. Hence, the number of distinct
	incoming port values among messages of a given $(\sigma, \tau_i)$-signature
	at $a$ equals precisely the number of distinct entities of type $\sigma$
	connected to $a$ via edge type $\tau_i$.
	
	\medskip
	\noindent\textsc{Layer 0 (Initialization).}
	Set $h^{(0)}(v) = \tau_V(v)$ for all $v \in V$.
	
	\medskip
	\noindent\textsc{Layer 1 (Forward: entity $\to$ attribute).}
	For each edge $(u, a, \tau) \in E$, define the message:
	\[
	\mathrm{msg}^{(1)}(u, a, \tau)
	= \bigl(\tau_V(u),\; \tau,\; p_{\mathrm{in}}(u, a, \tau)\bigr).
	\]
	Each attribute $a$ aggregates the multiset of incoming messages:
	\[
	\mathcal{M}(a) = \bigl\{\!\!\bigl\{
	\mathrm{msg}^{(1)}(u, a, \tau) : (u, a, \tau) \in E
	\bigr\}\!\!\bigr\}.
	\]
	By \cite[Lemma 5]{Xu-et-al:ICLR:2019}, $h^{(1)}(a)$ injectively encodes
	$\mathcal{M}(a)$. From this encoding, an MLP applied to $h^{(1)}(a)$
	can compute, for each entity type $\sigma \in \mathcal{T}_\mathcal{E}$
	and edge type $\tau_i \in \mathcal{T}_E$, the
	\emph{distinct-source count}:
	\[
	\mathrm{dsrc}_{\mathcal{M}(a)}(\sigma, \tau_i)
	= \bigl|\bigl\{ p \in \mathbb{N} :
	(\sigma, \tau_i, p) \in \mathcal{M}(a) \bigr\}\bigr|,
	\]
	which, by the port numbering convention, equals the number of distinct
	entities of type $\sigma$ connected to $a$ via edge type $\tau_i$.
	
	\medskip
	\noindent\textsc{Layer 2 (Reverse: attribute $\to$ entity).}
	For each edge $(u, a, \tau) \in E$, define the reverse message:
	\[
	\mathrm{msg}^{(2)}(u, a, \tau)
	= \bigl(h^{(1)}(a),\; \tau\bigr).
	\]
	Each entity $u$ aggregates the multiset of reverse messages:
	\[
	\mathcal{M}^{(2)}(u) = \bigl\{\!\!\bigl\{
	\mathrm{msg}^{(2)}(u, a, \tau) : (u, a, \tau) \in E
	\bigr\}\!\!\bigr\}.
	\]
	Each element of $\mathcal{M}^{(2)}(u)$ is a pair $(S, \tau')$, where
	$S$ is the injectively encoded multiset $\mathcal{M}(a)$ for the
	attribute $a$ at the other end of the edge, and $\tau'$ is the edge
	type. Again by \cite[Lemma 5]{Xu-et-al:ICLR:2019}, $h^{(2)}(u)$
	injectively encodes $\mathcal{M}^{(2)}(u)$.
	
	\medskip
	\noindent\textsc{Readout.}
	From $h^{(2)}(u)$, an MLP computes:
	\[
	\hat{y}(u) = \mathbf{1}\Bigl[\,\exists\; (S, \tau') \in
	\mathcal{M}^{(2)}(u) \;\text{such that}\;
	\mathrm{dsrc}_S\bigl(\tau_V(u),\, \tau'\bigr) \geq 2 \,\Bigr].
	\]
	
	\medskip
	\noindent\textsc{Correctness.}
	We show $\hat{y}(u) = \mathrm{Dup}(u)$ for all entities $u \in \mathcal{E}$.
	
	\medskip
	\noindent$(\Rightarrow)$ Suppose $\mathrm{Dup}(u) = 1$. Then there exist
	$v \neq u$ with $\tau_V(v) = \tau_V(u)$, an attribute
	$a \in \mathcal{A}$, and an edge type $\tau_i \in \mathcal{T}_E$ such
	that both edges $(u, a, \tau_i)$ and $(v, a, \tau_i)$ exist in $E$.
	Since $u \neq v$ are distinct sources into $a$, by property~(ii) of
	Adaptation~2a they receive distinct incoming port numbers at $a$, so
	$\mathrm{dsrc}_{\mathcal{M}(a)}(\tau_V(u), \tau_i) \geq 2$.
	Since the edge $(u, a, \tau_i)$ exists, the pair
	$(S, \tau') = (\mathcal{M}(a), \tau_i)$ belongs to
	$\mathcal{M}^{(2)}(u)$, and the readout condition is satisfied.
	Hence $\hat{y}(u) = 1$.
	
	\medskip
	\noindent$(\Leftarrow)$ Suppose $\hat{y}(u) = 1$. Then there exists
	$(S, \tau') \in \mathcal{M}^{(2)}(u)$ with
	$\mathrm{dsrc}_S(\tau_V(u), \tau') \geq 2$. This pair arises from some
	edge $(u, a, \tau') \in E$, with $S = \mathcal{M}(a)$. The condition
	$\mathrm{dsrc}_{\mathcal{M}(a)}(\tau_V(u), \tau') \geq 2$ means at
	least two distinct values of $p_{\mathrm{in}}$ appear among messages
	of signature $(\tau_V(u), \tau')$ in $\mathcal{M}(a)$. By
	property~(ii) of Adaptation~2a, distinct incoming port values
	correspond to distinct source entities. Hence there are at least two
	distinct entities of type $\tau_V(u)$ connected to $a$ via edge type
	$\tau'$. One is $u$ itself; the other is some $v \neq u$ with
	$\tau_V(v) = \tau_V(u)$. Therefore
	$(a, \tau') \in N_{\mathrm{typed}}(u) \cap N_{\mathrm{typed}}(v) \neq
	\emptyset$, so $\mathrm{Dup}(u) = 1$.
	
	\medskip
	Therefore $\hat{y}(u) = \mathrm{Dup}(u)$ for all entities 
	$u \in \mathcal{E}$.
\end{proof}

	On simple typed entity-attribute graphs, each entity contributes at most
	one edge of type $\tau_i$ to any given attribute, so all parallel edges
	from the same source are absent and every source has a unique port.
	Consequently, $\mathrm{dsrc}_{\mathcal{M}(a)}(\sigma, \tau_i) =
	\mathrm{mult}_{\mathcal{M}(a)}(\sigma, \tau_i)$, and the construction
	above reduces exactly to the sufficiency proof of
	Theorem~\ref{thm:K21-simple}.

	The sufficiency construction uses only incoming port numbers
	$p_{\mathrm{in}}$, at Layer~1, and does not require outgoing port
	numbers at any layer. For $K_{2,1}$ detection on multigraph typed
	entity-attribute graphs, the minimal sufficient architecture is
	therefore reverse message passing together with incoming port numbering
	alone.

\begin{Remark}[Tightness]\label{rem:tightness-K21-multigraph}
	Theorem~\ref{thm:K21-multigraph} is tight in several respects.
	
	\emph{Universality of both directions.}  The necessity in part~(a)
	holds for every MPNN with reverse message passing but without port
	numbering: no such architecture, regardless of depth or width, can
	compute $\mathrm{Dup}$ on all multigraph typed entity-attribute
	graphs.  The sufficiency in part~(b) holds for every multigraph
	typed entity-attribute graph: the $2$-layer construction computes
	$\mathrm{Dup}(u)$ correctly on all such graphs, under every valid
	incoming port numbering.
	
	\emph{Depth minimality.}  Two layers are necessary, by the same
	information-flow argument as in Remark~\ref{rem:depth-K21}.  After
	one layer of reverse message passing---even with incoming port
	numbering---entity~$u$ receives reverse messages carrying the
	layer-$0$ attribute embeddings, which encode only attribute types.
	The fan-in information at each attribute, which is essential for
	detecting shared attributes, is computed at layer~$1$ via forward
	aggregation and only becomes available to entities at layer~$2$ via
	reverse aggregation.
	
	\emph{Outgoing port numbering is unnecessary.}  The sufficiency
	proof uses only incoming port numbers in the layer-$1$ forward
	messages; outgoing port numbers appear nowhere in the construction.
	For $K_{2,1}$ detection, the sole role of port numbering is to let
	each attribute count its \emph{distinct} entity predecessors of each
	type, which is precisely the function of incoming ports.
	
	\emph{Role of incoming port numbering relative to simplicity.}
	Comparing Theorems~\ref{thm:K21-simple}
	and~\ref{thm:K21-multigraph}, incoming port numbering plays exactly
	the role that the simplicity assumption plays on simple graphs.  On
	a simple typed entity-attribute graph, each entity contributes at
	most one edge of a given type to any attribute, so the raw message
	multiplicity at an attribute equals the number of distinct entity
	sources.  On a multigraph, parallel edges from a single entity
	inflate the multiplicity without implying a second source.  Incoming
	port numbering restores the ability to count distinct sources: by
	property~(ii) of Adaptation~2a, edges from distinct entities receive
	distinct incoming ports, while by property~(i), parallel edges from
	the same entity share the same port.  Thus, incoming port numbering
	is the minimal adaptation that compensates for the loss of the
	simplicity assumption.
\end{Remark}

\subsection{$K_{2,r}$ Detection on Simple Entity-Attribute Graphs}
\label{sec:K2r-simple}

This section contains the main result of the paper. The $K_{2,r}$ co-reference predicate for $r \ge 2$ captures the stronger matching criterion that two entities of the same type share \emph{at least $r$} distinct typed attribute values.  While $K_{2,1}$ detection requires
only reverse message passing (Theorem~\ref{thm:K21-simple}), we show that
$K_{2,r}$ for $r \ge 2$ is provably harder: ego IDs are necessary even on
simple typed entity-attribute graphs, even when reverse message passing
and full directed multigraph port numbering are both available.

Before stating this result in full generality
(Theorem~\ref{thm:K2r-simple}), we illustrate the key difficulty through the
base case $r = 2$, where the construction is small enough to trace
completely by hand.

\begin{Example}[$K_{2,2}$ indistinguishability]
	\label{ex:K22-indistinguishable}
	Fix entity type $\sigma \in \mathcal{T}_{\mathcal{E}}$, two attribute types
	$\alpha_1, \alpha_2 \in \mathcal{T}_{\mathcal{A}}$, and two edge types
	$\tau_1, \tau_2 \in \mathcal{T}_E$, where edges to attributes of type
	$\alpha_i$ carry edge type $\tau_i$.  All entity nodes are initialised
	with the same feature vector (encoding type~$\sigma$); attribute nodes of
	type~$\alpha_i$ are initialised with the feature vector encoding
	type~$\alpha_i$.  We construct two simple typed entity-attribute graphs
	$G_1$ and $G_2$ that have identical degree statistics yet differ in their
	$\mathrm{Dup}_2$ label at entity~$u$; see Figure~\ref{fig:K22-example}.
	
	\begin{figure}[t]
		\centering
		\begin{tikzpicture}[
			node/.style  = {circle, draw, thick,
				minimum size=9mm, inner sep=1pt, font=\small},
			tlbl/.style  = {font=\scriptsize, text=gray!75!black},
			>={Stealth[length=6pt, width=4pt]},
			semithick
			]
			
			\node[font=\small\bfseries] at (2.0, 3.6)
			{$G_1\colon\mathrm{Dup}_2(u)=1$};
			
			\node[node, label={[tlbl]left:$\sigma$}] (G1u) at (0,  2.4) {$u$};
			\node[node, label={[tlbl]left:$\sigma$}] (G1v) at (0,  1.0) {$v$};
			\node[node, label={[tlbl]left:$\sigma$}] (G1w) at (0, -1.0) {$w$};
			\node[node, label={[tlbl]left:$\sigma$}] (G1x) at (0, -2.4) {$x$};
			
			\node[node, label={[tlbl]right:$\alpha_1$}] (G1a1) at (4,  2.4) {$a_1$};
			\node[node, label={[tlbl]right:$\alpha_2$}] (G1a2) at (4,  1.0) {$a_2$};
			\node[node, label={[tlbl]right:$\alpha_1$}] (G1a3) at (4, -1.0) {$a_3$};
			\node[node, label={[tlbl]right:$\alpha_2$}] (G1a4) at (4, -2.4) {$a_4$};
			
			\draw[->]         (G1u) to (G1a1);
			\draw[->, dashed] (G1u) to (G1a2);
			\draw[->]         (G1v) to (G1a1);
			\draw[->, dashed] (G1v) to (G1a2);
			\draw[->]         (G1w) to (G1a3);
			\draw[->, dashed] (G1w) to (G1a4);
			\draw[->]         (G1x) to (G1a3);
			\draw[->, dashed] (G1x) to (G1a4);
			
			\node[font=\small\bfseries] at (9.5, 3.6)
			{$G_2\colon\mathrm{Dup}_2(u)=0$};
			
			\node[node, label={[tlbl]left:$\sigma$}] (G2u)  at (7.5,  2.4) {$u$};
			\node[node, label={[tlbl]left:$\sigma$}] (G2v1) at (7.5,  0.8) {$v_1$};
			\node[node, label={[tlbl]left:$\sigma$}] (G2v2) at (7.5, -0.8) {$v_2$};
			\node[node, label={[tlbl]left:$\sigma$}] (G2v3) at (7.5, -2.4) {$v_3$};
			
			\node[node, label={[tlbl]right:$\alpha_1$}] (G2a1) at (11.5,  2.4) {$a_1$};
			\node[node, label={[tlbl]right:$\alpha_2$}] (G2a2) at (11.5,  0.8) {$a_2$};
			\node[node, label={[tlbl]right:$\alpha_1$}] (G2a3) at (11.5, -0.8) {$a_3$};
			\node[node, label={[tlbl]right:$\alpha_2$}] (G2a4) at (11.5, -2.4) {$a_4$};
			
			\draw[->]         (G2u)  to (G2a1);
			\draw[->, dashed] (G2u)  to (G2a2);
			\draw[->]         (G2v1) to (G2a1);
			\draw[->, dashed] (G2v1) to (G2a4);
			\draw[->]         (G2v2) to (G2a3);
			\draw[->, dashed] (G2v2) to (G2a2);
			\draw[->]         (G2v3) to (G2a3);
			\draw[->, dashed] (G2v3) to (G2a4);
			
			\draw[->]         (4.8, -3.6) -- (5.8, -3.6)
			node[right, font=\scriptsize] {$\tau_1$-edge};
			\draw[->, dashed] (7.3, -3.6) -- (8.3, -3.6)
			node[right, font=\scriptsize] {$\tau_2$-edge};
			
		\end{tikzpicture}
		\caption{The simple typed entity-attribute graphs $G_1$ (left) and
			$G_2$ (right) from Example~\ref{ex:K22-indistinguishable}.}
		\label{fig:K22-example}
	\end{figure}
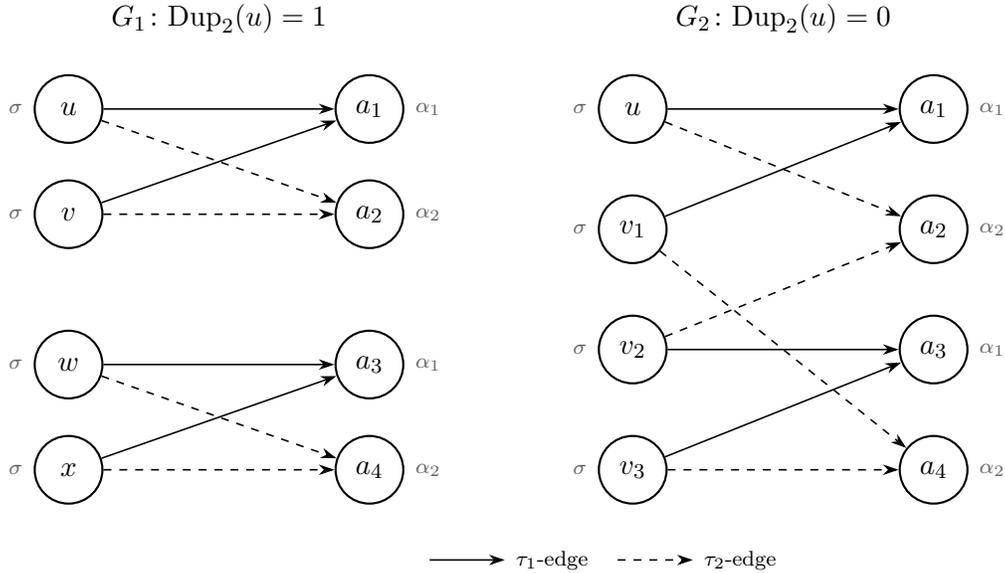

	Let $G_1$ be the typed graph consisting of entities $u, v, w, x$, all of type~$\sigma$, attributes $a_1, a_3$ of
	type~$\alpha_1$,  $a_2, a_4$ of type~$\alpha_2$, and the following typed edges: 
$(u, a_1, \tau_1)$, $(u, a_2, \tau_2)$, $(v, a_1, \tau_1)$, $(v, a_2, \tau_2)$, $(w, a_3, \tau_1)$, $(w, a_4, \tau_2)$, $(x, a_3, \tau_1)$, $(x, a_4, \tau_2)$. 
	Every entity has out-degree~$2$ and every attribute has in-degree~$2$.
	Since $N_{\mathrm{typed}}(u) = \{(a_1, \tau_1),\,(a_2, \tau_2)\} =
	N_{\mathrm{typed}}(v)$, we have $|N_{\mathrm{typed}}(u) \cap
	N_{\mathrm{typed}}(v)| = 2$, so $\mathrm{Dup}_2(u) = 1$.
	
Let $G_2$ be the typed graph consisting of entities $u, v_1, v_2, v_3$, all of type~$\sigma$, the same four attribute
	nodes, and the following typed edges: $(u, a_1, \tau_1)$, $(u, a_2, \tau_2)$, $(v_1, a_1, \tau_1)$, $(v_1, a_4, \tau_2)$, $(v_2, a_3, \tau_1)$, $(v_2, a_2, \tau_2)$, $(v_3, a_3, \tau_1)$, $(v_3, a_4, \tau_2)$. 
	Again, every entity has out-degree~$2$ and every attribute has
	in-degree~$2$.  The typed neighborhoods of $u$'s co-type entities are:
	\begin{align*}
		N_{\mathrm{typed}}(v_1) &= \{(a_1, \tau_1),\,(a_4, \tau_2)\}, \\
		N_{\mathrm{typed}}(v_2) &= \{(a_3, \tau_1),\,(a_2, \tau_2)\}, \\
		N_{\mathrm{typed}}(v_3) &= \{(a_3, \tau_1),\,(a_4, \tau_2)\}.
	\end{align*}
	Entity~$u$ shares exactly one typed neighbor with $v_1$ (namely
	$(a_1,\tau_1)$) and exactly one with $v_2$ (namely $(a_2,\tau_2)$), and
	none with $v_3$.  No single entity shares two or more typed neighbors with
	$u$, so $\mathrm{Dup}_2(u) = 0$.
	
	Both graphs are depicted in Figure~\ref{fig:K22-example}.  The two graphs
	have an identical \emph{degree-type profile}: 4 entities of
	type~$\sigma$, 2 attributes of type~$\alpha_1$, 2 attributes of
	type~$\alpha_2$, and 4 edges of each type, with every entity having
	out-degree~2 and every attribute having in-degree~2.  
	
	\textsc{Port numbering.} We exhibit specific valid port numberings on $G_1$ and $G_2$ under which the two graphs are indistinguishable; see Figure~\ref{fig:K22-example-ports}. 
	
	In $G_1$, the port numbering is chosen in a straightforward manner, as shown in Figure~\ref{fig:K22-example-ports}.  This is a valid port numbering for $G_1$ because every entity assigns distinct $p_{\mathrm{out}}$ values to its two outgoing edges, which go to distinct attributes, and every attribute assigns distinct $p_{\mathrm{in}}$ values
	to its two incoming edges, which come from distinct entities.
	
	Given this port numbering for $G_1$, the port numbering for $G_2$ is obtained by drawing the computation trees for $G_1$ and $G_2$ rooted at node $u$.  We ask: can we assign port numbers to $G_2$ so that these two computation trees are isomorphic, where isomorphism includes taking the edge labels of the computation trees -- namely the edge type and port numbers --- into account? The answer is yes, and this yields the port numbering for $G_2$ shown in Figure~\ref{fig:K22-example-ports}. 
	
	\begin{figure}[t]
		\centering
		\begin{tikzpicture}[
			node/.style  = {circle, draw, thick,
				minimum size=9mm, inner sep=1pt, font=\small},
			tlbl/.style  = {font=\scriptsize, text=gray!75!black},
			pout/.style  = {font=\scriptsize, red},
			pin/.style   = {font=\scriptsize, blue},
			>={Stealth[length=6pt, width=4pt]},
			semithick
			]
			
			\node[font=\small\bfseries] at (2.0, 3.6)
			{$G_1\colon\mathrm{Dup}_2(u)=1$};
			
			\node[node, label={[tlbl]left:$\sigma$}] (G1u) at (0,  2.4) {$u$};
			\node[node, label={[tlbl]left:$\sigma$}] (G1v) at (0,  1.0) {$v$};
			\node[node, label={[tlbl]left:$\sigma$}] (G1w) at (0, -1.0) {$w$};
			\node[node, label={[tlbl]left:$\sigma$}] (G1x) at (0, -2.4) {$x$};
			
			\node[node, label={[tlbl]right:$\alpha_1$}] (G1a1) at (4,  2.4) {$a_1$};
			\node[node, label={[tlbl]right:$\alpha_2$}] (G1a2) at (4,  1.0) {$a_2$};
			\node[node, label={[tlbl]right:$\alpha_1$}] (G1a3) at (4, -1.0) {$a_3$};
			\node[node, label={[tlbl]right:$\alpha_2$}] (G1a4) at (4, -2.4) {$a_4$};
			
			\draw[->] (G1u) to
			node[pos=0.15, above, pout] {$1$}
			node[pos=0.85, above, pin]  {$1$}
			(G1a1);
			
			\draw[->, dashed] (G1u) to
			node[pos=0.15, sloped, below, pout] {$2$}
			node[pos=0.85, sloped, above, pin]  {$1$}
			(G1a2);
			
			\draw[->] (G1v) to
			node[pos=0.05, sloped, above, pout] {$1$}
			node[pos=0.85, sloped, below, pin]  {$2$}
			(G1a1);
			
			\draw[->, dashed] (G1v) to
			node[pos=0.15, below, pout] {$2$}
			node[pos=0.85, below, pin]  {$2$}
			(G1a2);
			
			\draw[->] (G1w) to
			node[pos=0.15, above, pout] {$1$}
			node[pos=0.85, above, pin]  {$1$}
			(G1a3);
			
			\draw[->, dashed] (G1w) to
			node[pos=0.15, sloped, below, pout] {$2$}
			node[pos=0.85, sloped, above, pin]  {$1$}
			(G1a4);
			
			\draw[->] (G1x) to
			node[pos=0.05, sloped, above, pout] {$1$}
			node[pos=0.85, sloped, below, pin]  {$2$}
			(G1a3);
			
			\draw[->, dashed] (G1x) to
			node[pos=0.15, below, pout] {$2$}
			node[pos=0.85, below, pin]  {$2$}
			(G1a4);
			
			\node[font=\small\bfseries] at (9.5, 3.6)
			{$G_2\colon\mathrm{Dup}_2(u)=0$};
			
			\node[node, label={[tlbl]left:$\sigma$}] (G2u)  at (7.5,  2.4) {$u$};
			\node[node, label={[tlbl]left:$\sigma$}] (G2v1) at (7.5,  0.8) {$v_1$};
			\node[node, label={[tlbl]left:$\sigma$}] (G2v2) at (7.5, -0.8) {$v_2$};
			\node[node, label={[tlbl]left:$\sigma$}] (G2v3) at (7.5, -2.4) {$v_3$};
			
			\node[node, label={[tlbl]right:$\alpha_1$}] (G2a1) at (11.5,  2.4) {$a_1$};
			\node[node, label={[tlbl]right:$\alpha_2$}] (G2a2) at (11.5,  0.8) {$a_2$};
			\node[node, label={[tlbl]right:$\alpha_1$}] (G2a3) at (11.5, -0.8) {$a_3$};
			\node[node, label={[tlbl]right:$\alpha_2$}] (G2a4) at (11.5, -2.4) {$a_4$};
			
			\draw[->] (G2u) to
			node[pos=0.15, above, pout] {$1$}
			node[pos=0.85, above, pin]  {$1$}
			(G2a1);
			
			\draw[->, dashed] (G2u) to
			node[pos=0.15, sloped, below, pout] {$2$}
			node[pos=0.85, sloped, above, pin]  {$1$}
			(G2a2);
			
			\draw[->] (G2v1) to
			node[pos=0.15, sloped, below, pout] {$1$}
			node[pos=0.85, sloped, below, pin]  {$2$}
			(G2a1);
			
			\draw[->, dashed] (G2v1) to
			node[pos=0.12, sloped, above, pout] {$2$}
			node[pos=0.88, sloped, below, pin]  {$2$}
			(G2a4);
			
			\draw[->] (G2v2) to
			node[pos=0.15, below, pout] {$1$}
			node[pos=0.85, above, pin]  {$2$}
			(G2a3);
			
			\draw[->, dashed] (G2v2) to
			node[pos=0.10, sloped, above, pout] {$2$}
			node[pos=0.85, sloped, below, pin]  {$2$}
			(G2a2);
			
			\draw[->] (G2v3) to
			node[pos=0.15, sloped, above, pout] {$1$}
			node[pos=0.85, sloped, below, pin]  {$1$}
			(G2a3);
			
			\draw[->, dashed] (G2v3) to
			node[pos=0.15, below, pout] {$2$}
			node[pos=0.85, below, pin]  {$1$}
			(G2a4);
			
			\draw[->]        (2.2, -3.6) -- (3.2, -3.6)
			node[right, font=\scriptsize] {$\tau_1$-edge (solid)};
			\draw[->, dashed] (2.2, -4.1) -- (3.2, -4.1)
			node[right, font=\scriptsize] {$\tau_2$-edge (dashed)};
			\node[font=\scriptsize] at (9.1, -3.85)
			{\textcolor{red}{red} near tail $= p_{\mathrm{out}}$;\quad
				\textcolor{blue}{blue} near head $= p_{\mathrm{in}}$};
			
		\end{tikzpicture}
		\caption{A port numbering for the simple typed entity-attribute graphs $G_1$ and $G_2$ from Example~\ref{ex:K22-indistinguishable}.  \textcolor{red}{Red} numbers near the
			source of an edge give $p_{\mathrm{out}}$; \textcolor{blue}{blue} numbers near the
			destination of an edge give $p_{\mathrm{in}}$.  Both port numberings are valid: each
			entity assigns distinct outgoing ports $\{1,2\}$ to its two edges, and
			each attribute receives distinct incoming ports $\{1,2\}$ from its two
			predecessors.  Under these port numberings, every MPNN without ego IDs assigns $u$
			the identical embedding in both graphs.}
		\label{fig:K22-example-ports}
	\end{figure}
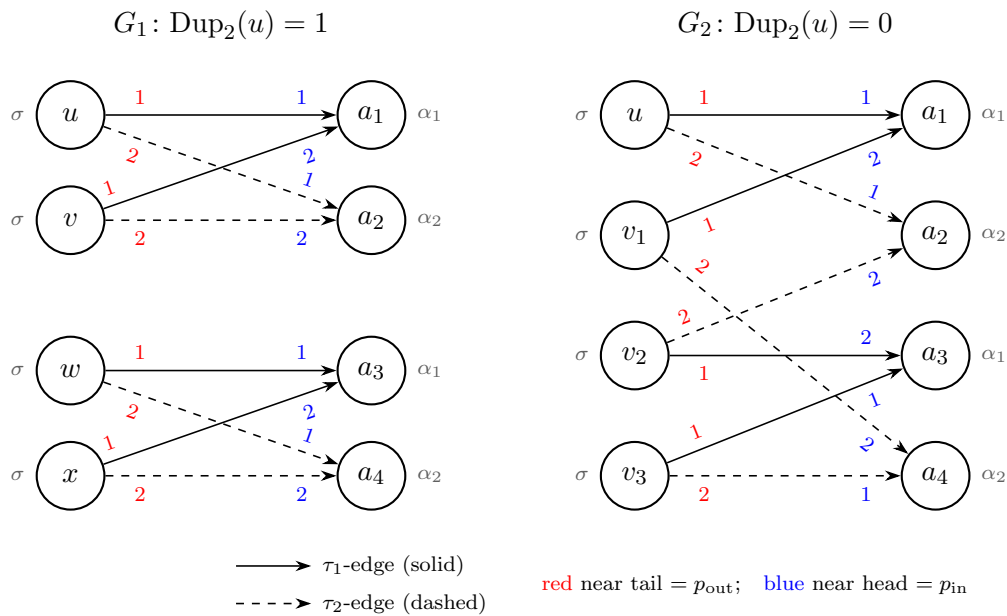

	\textsc{Two-class partition.}
	Under these port numberings, the entities in each graph partition into two
	classes according to their incoming port at every attribute they connect to:
	\begin{itemize}
		\item \emph{Class-1 entities} ($p_{\mathrm{in}} = 1$ at every incident
		attribute): $\{u, w\}$ in $G_1$;\; $\{u, v_3\}$ in $G_2$.
		\item \emph{Class-2 entities} ($p_{\mathrm{in}} = 2$ at every incident
		attribute): $\{v, x\}$ in $G_1$;\; $\{v_1, v_2\}$ in $G_2$.
	\end{itemize}
	In both graphs, at every attribute the class-1 predecessor has
	$p_{\mathrm{in}} = 1$ and the class-2 predecessor has $p_{\mathrm{in}} =
	2$; and every entity uses $p_{\mathrm{out}} = 1$ for its $\tau_1$-edge
	and $p_{\mathrm{out}} = 2$ for its $\tau_2$-edge.
	
	\textsc{Indistinguishability.}
	One can prove by induction on the layer index~$t$ that at every depth:
	\begin{enumerate}
		\item[(i)] all class-1 entities share a common embedding $h_1^{(t)}$,
		the same value in both $G_1$ and $G_2$;
		\item[(ii)] all class-2 entities share a common embedding $h_2^{(t)}$,
		the same value in both $G_1$ and $G_2$;
		\item[(iii)] all type-$\alpha_i$ attributes share a common embedding
		$h_{\alpha_i}^{(t)}$, the same value in both $G_1$ and $G_2$.
	\end{enumerate}
	The inductive step rests on a single observation: every type-$\alpha_i$
	attribute in both graphs receives, at each layer, exactly the same
	forward-message multiset --- one message from a class-1 entity (with
	embedding $h_1^{(t-1)}$, edge type~$\tau_i$, $p_{\mathrm{in}}=1$,
	$p_{\mathrm{out}}=i$) and one from a class-2 entity (with $h_2^{(t-1)}$,
	$\tau_i$, $p_{\mathrm{in}}=2$, $p_{\mathrm{out}}=i$).  Symmetrically,
	every class-$c$ entity receives the same reverse-message multiset,
	because it connects to one type-$\alpha_1$ attribute and one
	type-$\alpha_2$ attribute (both with the same type-specific embeddings in
	both graphs).
	
	Since $u$ is a class-1 entity in both $G_1$ and $G_2$, it is assigned
	the identical embedding $h_1^{(K)}$ at every depth~$K$, yet
	$\mathrm{Dup}_2(u) = 1$ in $G_1$ and $\mathrm{Dup}_2(u) = 0$ in $G_2$.
	No MPNN readout function can separate these two cases. 
\end{Example}

We now present the main result of the paper:

\begin{Theorem}\label{thm:K2r-simple}
	Let $r \geq 2$. Consider the $K_{2,r}$ co-reference predicate
	$\mathrm{Dup}_r$ on simple typed entity-attribute graphs.
	\begin{enumerate}
		\item[(a)] \textbf{(Necessity of ego IDs.)} There exist simple typed
		entity-attribute graphs $G_1$ and $G_2$ and an entity $u$ such that
		$\mathrm{Dup}_{r,G_1}(u) = 1$ and $\mathrm{Dup}_{r,G_2}(u) = 0$, yet
		there exists a valid port numbering under which any MPNN of arbitrary
		depth equipped with reverse message passing and full directed multigraph
		port numbering (Adaptation~2) but \emph{without} ego IDs assigns
		identical embeddings to $u$ in both graphs.
		\item[(b)] \textbf{(Sufficiency of ego IDs.)} A $4$-layer MPNN with
		ego IDs (Adaptation~3) and reverse message passing (Adaptation~1) can
		compute $\mathrm{Dup}_r(u)$ for all entities $u \in \mathcal{E}$ in
		any simple typed entity-attribute graph.
	\end{enumerate}
	Consequently, the transition from $K_{2,1}$ to $K_{2,r}$ (for
	$r \geq 2$) incurs a strict increase in the required architectural
	complexity: ego IDs become necessary even on simple, acyclic, bipartite
	typed entity-attribute graphs.
\end{Theorem}

\begin{proof}[Proof of part (a): necessity]
	We construct two simple typed entity-attribute graphs that are
	indistinguishable to any MPNN with reverse message passing and full port
	numbering (but without ego IDs), yet differ in their $\mathrm{Dup}_r$
	labels. 
	
	\medskip
	\noindent\textsc{Construction.}
	Fix an entity type $\sigma \in \mathcal{T}_\mathcal{E}$, attribute types
	$\alpha_1, \ldots, \alpha_r \in \mathcal{T}_\mathcal{A}$, and edge types
	$\tau_1, \ldots, \tau_r \in \mathcal{T}_E$, where edges to
	type-$\alpha_j$ attributes use edge type $\tau_j$.
	
	In $G_1$, the \emph{positive instance}, the entity set is
	$\{u, v, w, x\}$, all of type $\sigma$. For each
	$j \in \{1,\ldots,r\}$, there are two attributes $a_j$ and $b_j$,
	both of type $\alpha_j$, giving $2r$ attributes in total. The typed edges of $G_1$
	are (see Figure~\ref{fig:construction:K2r}):
	\begin{align*}
		&(u, a_j, \tau_j) \;\text{ and }\; (v, a_j, \tau_j)
		\quad\text{for } j = 1, \ldots, r, \\
		&(w, b_j, \tau_j) \;\text{ and }\; (x, b_j, \tau_j)
		\quad\text{for } j = 1, \ldots, r.
	\end{align*}
	Every entity has out-degree $r$; every attribute has in-degree $2$;
	there are $4r$ edges in total. Since $u$ and $v$ share all $r$
	attributes $\{(a_j, \tau_j) : j = 1,\ldots,r\}$, we have
	$|N_{\mathrm{typed}}(u) \cap N_{\mathrm{typed}}(v)| = r$, so
	$\mathrm{Dup}_{r,G_1}(u) = 1$.
	
	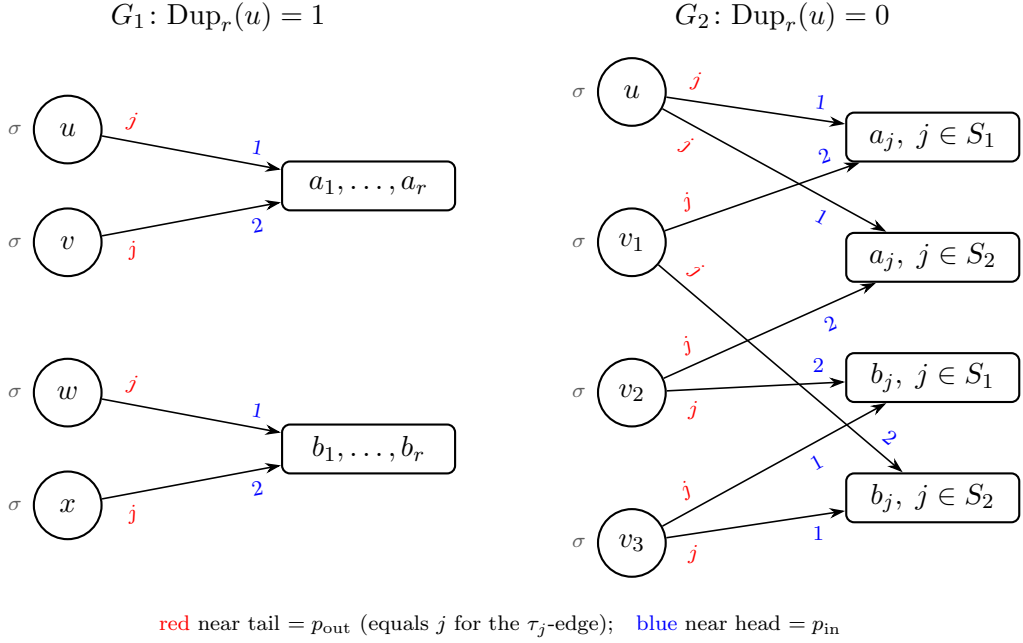
\begin{figure}[t]
		\centering
		\begin{tikzpicture}[
			node/.style  = {circle, draw, thick,
				minimum size=9mm, inner sep=1pt, font=\small},
			abox/.style  = {draw, thick, rounded corners=3pt,
				minimum width=2.3cm, minimum height=0.65cm,
				inner sep=3pt, font=\small},
			tlbl/.style  = {font=\scriptsize, text=gray!75!black},
			pout/.style  = {font=\scriptsize, red},
			pin/.style   = {font=\scriptsize, blue},
			>={Stealth[length=6pt, width=4pt]},
			semithick
			]
			
			\node[font=\small\bfseries] at (2.0, 4.0)
			{$G_1\colon\mathrm{Dup}_r(u)=1$};
			
			\node[node, label={[tlbl]left:$\sigma$}] (G1u) at (0,  2.5) {$u$};
			\node[node, label={[tlbl]left:$\sigma$}] (G1v) at (0,  1.0) {$v$};
			\node[abox] (G1A) at (4.0, 1.75) {$a_1, \ldots, a_r$};
			
			\node[node, label={[tlbl]left:$\sigma$}] (G1w) at (0, -1.0) {$w$};
			\node[node, label={[tlbl]left:$\sigma$}] (G1x) at (0, -2.5) {$x$};
			\node[abox] (G1B) at (4.0,-1.75) {$b_1, \ldots, b_r$};
			
			\draw[->] (G1u) to
			node[pos=0.15, sloped, above, pout] {$j$}
			node[pos=0.85, sloped, above, pin]  {$1$}
			(G1A);
			\draw[->] (G1v) to
			node[pos=0.15, sloped, below, pout] {$j$}
			node[pos=0.85, sloped, below, pin]  {$2$}
			(G1A);
			
			\draw[->] (G1w) to
			node[pos=0.15, sloped, above, pout] {$j$}
			node[pos=0.85, sloped, above, pin]  {$1$}
			(G1B);
			\draw[->] (G1x) to
			node[pos=0.15, sloped, below, pout] {$j$}
			node[pos=0.85, sloped, below, pin]  {$2$}
			(G1B);
			
			\node[font=\small\bfseries] at (9.5, 4.0)
			{$G_2\colon\mathrm{Dup}_r(u)=0$};
			
			\node[node, label={[tlbl]left:$\sigma$}] (G2u)  at (7.5,  3.0) {$u$};
			\node[node, label={[tlbl]left:$\sigma$}] (G2v1) at (7.5,  1.0) {$v_1$};
			\node[node, label={[tlbl]left:$\sigma$}] (G2v2) at (7.5, -1.0) {$v_2$};
			\node[node, label={[tlbl]left:$\sigma$}] (G2v3) at (7.5, -3.0) {$v_3$};
			
			\node[abox] (G2AS1) at (11.5,  2.4) {$a_j,\;j\in S_1$};
			\node[abox] (G2AS2) at (11.5,  0.8) {$a_j,\;j\in S_2$};
			\node[abox] (G2BS1) at (11.5, -0.8) {$b_j,\;j\in S_1$};
			\node[abox] (G2BS2) at (11.5, -2.4) {$b_j,\;j\in S_2$};
			
			\draw[->] (G2u) to
			node[pos=0.15, sloped, above, pout] {$j$}
			node[pos=0.85, sloped, above, pin]  {$1$}
			(G2AS1);
			
			\draw[->] (G2u) to
			node[pos=0.15, sloped, below, pout] {$j$}
			node[pos=0.75, sloped, below, pin]  {$1$}
			(G2AS2);
			
			\draw[->] (G2v1) to
			node[pos=0.15, sloped, above, pout] {$j$}
			node[pos=0.85, sloped, above, pin]  {$2$}
			(G2AS1);
			
			\draw[->] (G2v1) to
			node[pos=0.10, sloped, above, pout] {$j$}
			node[pos=0.90, sloped, above, pin]  {$2$}
			(G2BS2);
			
			\draw[->] (G2v2) to
			node[pos=0.15, sloped, above, pout] {$j$}
			node[pos=0.75, sloped, below, pin]  {$2$}
			(G2AS2);
			
			\draw[->] (G2v2) to
			node[pos=0.15, below, pout] {$j$}
			node[pos=0.85, above, pin]  {$2$}
			(G2BS1);
			
			\draw[->] (G2v3) to
			node[pos=0.15, sloped, above, pout] {$j$}
			node[pos=0.65, sloped, below, pin]  {$1$}
			(G2BS1);
			
			\draw[->] (G2v3) to
			node[pos=0.15, below, pout] {$j$}
			node[pos=0.85, below, pin]  {$1$}
			(G2BS2);
			
			\node[font=\scriptsize] at (5.75, -4.1)
			{\textcolor{red}{red} near tail $= p_{\mathrm{out}}$
				(equals $j$ for the $\tau_j$-edge);\quad
				\textcolor{blue}{blue} near head $= p_{\mathrm{in}}$};
			
		\end{tikzpicture}
		\caption{Port numbering for the general $K_{2,r}$ case
			($r \geq 2$), generalizing Figure~\ref{fig:K22-example-ports}.
			\emph{Left ($G_1$):} two disjoint copies of $K_{2,r}$; the
			top copy has entities $u$ (class-1) and $v$ (class-2) sharing
			all $r$ typed attributes $a_1,\ldots,a_r$, giving
			$\mathrm{Dup}_r(u)=1$.
			\emph{Right ($G_2$):} the four attribute boxes partition
			$\{1,\ldots,r\} = S_1 \dot{\cup} S_2$ (with $S_1, S_2 \neq
			\emptyset$); entity $u$'s $r$ attributes are split across
			$v_1$ (covering $S_1$-indices) and $v_2$ (covering
			$S_2$-indices), so no single entity shares $r$ typed neighbors
			with $u$, giving $\mathrm{Dup}_r(u)=0$.
			Each bundled arrow represents $|S_i|$ parallel edges, one per
			index $j$ in the relevant set.
			\textcolor{red}{Red} near the tail gives $p_{\mathrm{out}}$
			(equal to $j$ for the $\tau_j$-edge); \textcolor{blue}{blue}
			near the head gives $p_{\mathrm{in}}$.
			Class-1 entities ($u$, $w$ in $G_1$; $u$, $v_3$ in $G_2$)
			receive $p_{\mathrm{in}}=1$ at every incident attribute box;
			class-2 entities ($v$, $x$; $v_1$, $v_2$) receive
			$p_{\mathrm{in}}=2$.  The computation trees rooted at $u$ are
			therefore isomorphic in $G_1$ and $G_2$, making $u$
			indistinguishable to any MPNN without ego IDs.}
		\label{fig:construction:K2r}
	\end{figure}
	
	In $G_2$, the \emph{negative instance}, the entity set is
	$\{u, v_1, v_2, v_3\}$, all of type $\sigma$. The attribute set is
	the same: $a_j, b_j$ of type $\alpha_j$ for each $j$. Fix a partition
	$\{1,\ldots,r\} = S_1 \dot{\cup} S_2$ with $S_1, S_2 \neq \emptyset$; for concreteness, take $S_1 = \{1\}$ and $S_2 = \{2,\ldots,r\}$. The
	typed edges of $G_2$ are (see Figure~\ref{fig:construction:K2r}):
	\begin{align*}
		&(u,\, a_j,\, \tau_j) \quad\text{for all } j = 1, \ldots, r, \\
		&(v_1,\, a_j,\, \tau_j) \text{ for } j \in S_1;\quad
		(v_1,\, b_j,\, \tau_j) \text{ for } j \in S_2, \\
		&(v_2,\, a_j,\, \tau_j) \text{ for } j \in S_2;\quad
		(v_2,\, b_j,\, \tau_j) \text{ for } j \in S_1, \\
		&(v_3,\, b_j,\, \tau_j) \quad\text{for all } j = 1, \ldots, r.
	\end{align*}
	Every entity again has out-degree $r$; every attribute has in-degree
	$2$; there are $4r$ edges. Entity $u$'s attributes are ``split'' across
	multiple entities:
	$|N_{\mathrm{typed}}(u) \cap N_{\mathrm{typed}}(v_1)| = |S_1| \leq r-1$,
	$|N_{\mathrm{typed}}(u) \cap N_{\mathrm{typed}}(v_2)| = |S_2| \leq r-1$,
	and
	$|N_{\mathrm{typed}}(u) \cap N_{\mathrm{typed}}(v_3)| = 0$.
	No single entity shares $r$ or more attributes with $u$, so
	$\mathrm{Dup}_{r,G_2}(u) = 0$.
	
	Both graphs are simple typed entity-attribute graphs with $4$ entities,
	$2r$ attributes, and $4r$ edges.
	
	\medskip
	\noindent\textsc{Port numbering.}
	For each entity in both graphs, assign outgoing port $p_{\mathrm{out}} = j$
	to its $\tau_j$-edge. At each attribute, the two predecessors are
	partitioned into two classes: a \emph{class-1} predecessor receives
	$p_{\mathrm{in}} = 1$ and a \emph{class-2} predecessor receives
	$p_{\mathrm{in}} = 2$.
	
	In $G_1$: class-1 entities are $\{u, w\}$ (receiving $p_{\mathrm{in}} = 1$
	at all their incident attributes) and class-2 entities are $\{v, x\}$
	(receiving $p_{\mathrm{in}} = 2$). Explicitly:
	\begin{center}
		\begin{tabular}{lcc}
			\textbf{Edge} & $p_{\mathrm{out}}$ & $p_{\mathrm{in}}$ \\
			\hline
			$(u, a_j, \tau_j)$ & $j$ & $1$ \\
			$(v, a_j, \tau_j)$ & $j$ & $2$ \\
			$(w, b_j, \tau_j)$ & $j$ & $1$ \\
			$(x, b_j, \tau_j)$ & $j$ & $2$ \\
		\end{tabular}
	\end{center}
	
	In $G_2$: class-1 entities are $\{u, v_3\}$ and class-2 entities are
	$\{v_1, v_2\}$. Explicitly:
	\begin{center}
		\begin{tabular}{lcc}
			\textbf{Edge} & $p_{\mathrm{out}}$ & $p_{\mathrm{in}}$ \\
			\hline
			$(u, a_j, \tau_j)$ & $j$ & $1$ \\
			$(v_1, a_j, \tau_j)$ for $j \in S_1$ & $j$ & $2$ \\
			$(v_1, b_j, \tau_j)$ for $j \in S_2$ & $j$ & $2$ \\
			$(v_2, a_j, \tau_j)$ for $j \in S_2$ & $j$ & $2$ \\
			$(v_2, b_j, \tau_j)$ for $j \in S_1$ & $j$ & $2$ \\
			$(v_3, b_j, \tau_j)$ & $j$ & $1$ \\
		\end{tabular}
	\end{center}
	
	\emph{Validity}: In both graphs, each entity has $r$ edges to $r$
	distinct attributes, assigned $p_{\mathrm{out}} = 1, \ldots, r$
	(distinct outgoing ports). At each attribute, the two predecessors have
	$p_{\mathrm{in}} \in \{1, 2\}$ (distinct incoming ports). This
	satisfies properties~(i)--(iii) of Adaptation~2.
	
	\medskip
	\noindent\textsc{Indistinguishability.}
	We prove that for all depths $k \geq 0$, the following three
	conditions hold simultaneously in $G_1$ and $G_2$:
	\begin{enumerate}
		\item[(I)] All class-1 entities share a common embedding $h_1^{(k)}$,
		the same in both graphs.
		\item[(II)] All class-2 entities share a common embedding $h_2^{(k)}$,
		the same in both graphs.
		\item[(III)] For each $j \in \{1,\ldots,r\}$, all type-$\alpha_j$
		attributes share a common embedding $h_{\alpha_j}^{(k)}$, the same
		in both graphs.
	\end{enumerate}
	Note that unlike the simpler indistinguishability arguments for
	Theorems~\ref{thm:K21-simple} and~\ref{thm:K21-multigraph}, we do
	\emph{not} claim all entity nodes share a single common embedding. The
	incoming port numbers cause entities to split into two classes with
	potentially different embeddings $h_1^{(k)} \neq h_2^{(k)}$. The key
	insight is that this two-class structure is preserved \emph{identically}
	across both graphs.
	
	\emph{Base case} ($k = 0$): All entities embed as $\sigma$, with no ego
	feature; type-$\alpha_j$ attributes embed as $\alpha_j$. Conditions
	(I)--(III) hold.
	
	\emph{Inductive step --- forward aggregation at type-$\alpha_j$
		attributes}: Each type-$\alpha_j$ attribute has exactly $2$ incoming
	edges, both of type $\tau_j$: one from a class-1 entity
	($p_{\mathrm{in}} = 1$) and one from a class-2 entity
	($p_{\mathrm{in}} = 2$). The forward-message multiset at every
	type-$\alpha_j$ attribute is therefore:
	\[
	\bigl\{\!\!\bigl\{
	(h_1^{(k-1)},\, \tau_j,\, 1,\, j),\;\;
	(h_2^{(k-1)},\, \tau_j,\, 2,\, j)
	\bigr\}\!\!\bigr\},
	\]
	where each message is a tuple of (entity embedding, edge type, incoming
	port, outgoing port). This multiset is the same at every type-$\alpha_j$
	attribute in both graphs, because for each $j$ and each type-$\alpha_j$
	attribute in either graph, the two predecessors consist of exactly one
	class-1 and one class-2 entity. The update yields the same
	$h_{\alpha_j}^{(k)}$ everywhere, establishing~(III).
	
	\emph{Inductive step --- reverse aggregation at class-1 entities}: A
	class-1 entity has $r$ outgoing edges: for each $j$, one $\tau_j$-edge to
	some type-$\alpha_j$ attribute, with $p_{\mathrm{out}} = j$ and
	$p_{\mathrm{in}} = 1$. The reverse-message multiset is:
	\[
	\bigl\{\!\!\bigl\{
	(h_{\alpha_j}^{(k-1)},\, \tau_j,\, 1,\, j)
	: j = 1,\ldots,r
	\bigr\}\!\!\bigr\}.
	\]
	Since $h^{(k-1)}(a_j) = h^{(k-1)}(b_j) = h_{\alpha_j}^{(k-1)}$ by the
	inductive hypothesis, all class-1 entities in both graphs produce the
	same multiset. This establishes~(I).
	
	\emph{Inductive step --- reverse aggregation at class-2 entities}:
	Identical argument with $p_{\mathrm{in}} = 2$. The reverse-message
	multiset is:
	\[
	\bigl\{\!\!\bigl\{
	(h_{\alpha_j}^{(k-1)},\, \tau_j,\, 2,\, j)
	: j = 1,\ldots,r
	\bigr\}\!\!\bigr\}.
	\]
	This establishes~(II) and completes the inductive step.
	
	\medskip
	\noindent\textsc{Conclusion.}
	Since $u$ is a class-1 entity in both graphs, $h^{(K)}(u) = h_1^{(K)}$
	in both $G_1$ and $G_2$ for every depth $K \geq 0$. But
	$\mathrm{Dup}_{r,G_1}(u) = 1$ and $\mathrm{Dup}_{r,G_2}(u) = 0$.
	Therefore, no MPNN with reverse message passing and full port numbering,
	but without ego IDs, can compute $\mathrm{Dup}_r$ on all simple typed
	entity-attribute graphs.
\end{proof}

\noindent We briefly discuss the intuition on why ego IDs are needed. The fundamental
obstacle is \emph{cross-attribute identity correlation}. For
$\mathrm{Dup}_1$, each attribute independently determines whether it has
two same-type entity predecessors---a purely local computation requiring
no identity tracking. For $\mathrm{Dup}_r$ with $r \geq 2$, one must
verify that the \emph{same} entity $v$ appears at $r$ different attributes
of entity $u$. This requires correlating information across different
attributes, which in turn requires distinguishing the entity originating a
message at one attribute from entities at another. Without ego IDs, all
same-type entities in the same port-numbering class look identical, and the
required correlation is impossible.

\begin{proof}[Proof of part (b): sufficiency]
	We construct an explicit $4$-layer MPNN with ego IDs (Adaptation~3) and
	reverse message passing (Adaptation~1) that computes $\mathrm{Dup}_r(u)$
	for a given ego entity $u \in \mathcal{E}$ in any simple typed
	entity-attribute graph. The construction does not use port numbering.
	
	Rather than tracking full injective multiset encodings at every layer, we
	specify the exact quantities each node computes and stores (see Table~\ref{tab:4layer-construction} for a summary), and verify at
	each layer that the computation is realizable within the MPNN framework.
	
	\medskip
	\noindent\textsc{Layer 1 (Forward: entity $\to$ attribute).}
	Each attribute $a$ receives, along each incoming edge $(v, a, \tau) \in E$,
	the message:
	\[
	\mathrm{msg}^{(1)}(v, a, \tau)
	= \bigl((\tau_V(v),\, \mathrm{ego}_u(v)),\;\, \tau\bigr).
	\]
	The first component of the pair encodes the sender's type and its ego
	indicator $\mathrm{ego}_u(v) \in \{0,1\}$, defined at Layer~0 and equal to
	$1$ if and only if $v = u$.  From these messages, attribute $a$ computes
	and stores:
	\[
	\mathrm{EgoTypes}(a)
	:= \bigl\{\tau \in \mathcal{T}_E :
	\text{some message into } a \text{ has } \mathrm{ego}_u(v) = 1
	\text{ and edge type } \tau\bigr\}.
	\]
	This is the set of edge types via which the ego entity $u$ connects to
	$a$.
	
\emph{Embedding.}
Attribute $a$ stores $h^{(1)}(a) = \bigl(\tau_V(a),\;
\mathrm{EgoTypes}(a)\bigr)$, retaining its own type alongside the
set of edge types via which the ego connects to it.

\emph{Realizability.}
We show that $\mathrm{EgoTypes}(a)$ can be computed by a sum
aggregation followed by an MLP, and hence is realizable within the
MPNN framework.  Define the per-message function
\[
f\bigl((\sigma,\, \mathit{fl}),\, \tau\bigr)
\;=\;
\mathit{fl} \cdot \mathbf{e}_\tau
\;\in\; \{0,1\}^{|\mathcal{T}_E|},
\]
where $\sigma = \tau_V(v)$ is the sender's entity type, ego-flag
$\mathit{fl} = \mathrm{ego}_u(v) \in \{0,1\}$ is the sender's ego
indicator (equal to $1$ iff $v = u$), and $\mathbf{e}_\tau$ is the
standard basis vector for edge type~$\tau$.  The function $f$ returns
$\mathbf{e}_\tau$ for a message from the ego (whose ego-flag $\mathit{fl}=1$) and the
zero vector for any non-ego message, where $\mathit{fl}=0$, so it acts as
a gate that discards all non-ego contributions.  Summing over all
incoming edges to $a$:
\[
\sum_{v \in N_{\mathrm{in}}(a)}\;
\sum_{\tau\,:\,(v,a,\tau)\in E}
f\bigl(\mathrm{msg}^{(1)}(v,a,\tau)\bigr)
\;=\;
\sum_{\tau\,:\,(u,a,\tau)\in E}
\mathbf{e}_\tau,
\]
which is a vector whose $\tau$-th component equals the number of
$\tau$-edges from $u$ to $a$, which is at most $1$ since the graph is simple.
Applying a component-wise threshold recovers
$\mathrm{EgoTypes}(a) = \{\tau : (u,a,\tau) \in E\}$, which is a
deterministic function of the sum and hence realizable by an MLP of
sufficient width.

\emph{Derived predicate.}
Before proceeding to Layer~2, we extract from $h^{(1)}(a)$ a single
binary quantity that will allow each entity to count how many of its
own attribute neighbors are also neighbors of the ego.  Specifically,
for any attribute embedding $h^{(1)}(a)$ and edge type $\tau$, define:
\[
\mathrm{egoEdge}\bigl(h^{(1)}(a),\, \tau\bigr)
:= \mathbf{1}\bigl[\tau \in \mathrm{EgoTypes}(a)\bigr].
\]
This equals $1$ if and only if the ego entity $u$ has an edge of
type~$\tau$ to attribute~$a$, and is computable directly from
$h^{(1)}(a)$ since $\mathrm{EgoTypes}(a)$ is stored there.

\medskip
\noindent\textsc{Layer 2 (Reverse: attribute $\to$ entity).}
Each entity $v$ receives, along each outgoing edge $(v, a, \tau) \in E$,
the reverse message:
\[
\mathrm{msg}^{(2)}(v, a, \tau)
= \bigl(h^{(1)}(a),\; \tau\bigr).
\]
From $(h^{(1)}(a), \tau)$, the receiving entity can evaluate
$\mathrm{egoEdge}(h^{(1)}(a), \tau) \in \{0,1\}$, which indicates
whether the ego $u$ also uses a $\tau$-edge to reach $a$.

Entity $v$ computes the \emph{ego-overlap count}:
\[
\mathrm{ov}(v)
:= \sum_{(v,\, a,\, \tau) \in E}
\mathrm{egoEdge}\bigl(h^{(1)}(a),\; \tau\bigr).
\]
This counts the number of $v$'s outgoing edges $(v, a, \tau)$ such that
the ego $u$ also has a $\tau$-edge to $a$, i.e., the number of typed
attribute neighbors shared between $v$ and $u$.

\emph{Embedding.}
Entity $v$ stores $h^{(2)}(v) = \bigl(\tau_V(v),\; \mathrm{ego}_u(v),\;
\mathrm{ov}(v)\bigr)$, retaining its type and ego indicator from
Layer~0 alongside the newly computed overlap count.

\emph{Realizability.}
We show that $\mathrm{ov}(v)$ is computable by a sum aggregation
followed by an MLP, and hence is realizable within the MPNN framework.
Define the per-message function $\varphi$ by
\[
\varphi\bigl(\mathrm{msg}^{(2)}(v,a,\tau)\bigr)
\;=\;
\mathrm{egoEdge}\bigl(h^{(1)}(a),\, \tau\bigr) \;\in\; \{0,1\}.
\]
Summing $\varphi$ over all reverse messages received by $v$ then
yields $\mathrm{ov}(v)$ directly:
\[
\sum_{a \in N_{\mathrm{out}}(v)}\;
\sum_{\tau\,:\,(v,a,\tau)\in E}
\varphi\bigl(\mathrm{msg}^{(2)}(v,a,\tau)\bigr)
\;=\;
\sum_{a \in N_{\mathrm{out}}(v)}\;
\sum_{\tau\,:\,(v,a,\tau)\in E}
\mathrm{egoEdge}\bigl(h^{(1)}(a),\,\tau\bigr)
\;=\;
\mathrm{ov}(v).
\]
The update function then takes as input this sum together with $v$'s
own previous embedding $h^{(1)}(v) = h^{(0)}(v) = (\tau_V(v),
\mathrm{ego}_u(v))$ — available to every node at every layer as part
of the standard MPNN update — and outputs the triple
$(\tau_V(v), \mathrm{ego}_u(v), \mathrm{ov}(v))$.
	
	\begin{Lemma}[Semantic correctness of $\mathrm{ov}$]
		\label{lem:ov-correct}
		For every entity $v \in \mathcal{E}$ in a simple typed entity-attribute
		graph,
		$\mathrm{ov}(v) = |N_{\mathrm{typed}}(u) \cap N_{\mathrm{typed}}(v)|$.
	\end{Lemma}
	
	\begin{proof}
		Each outgoing edge $(v, a, \tau)$ from $v$ certifies
		$(a, \tau) \in N_{\mathrm{typed}}(v)$. The term
		$\mathrm{egoEdge}(h^{(1)}(a), \tau) = 1$ iff
		$(a, \tau) \in N_{\mathrm{typed}}(u)$. So each edge contributes $1$ to
		the sum iff $(a, \tau) \in N_{\mathrm{typed}}(u) \cap
		N_{\mathrm{typed}}(v)$. In a simple typed entity-attribute graph, for
		each pair $(a, \tau)$ there is at most one edge from $v$ to $a$ of type
		$\tau$, so distinct elements of $N_{\mathrm{typed}}(v)$ correspond to
		distinct edges, and each is counted exactly once.
	\end{proof}
	
	\medskip
	\noindent\textsc{Layer 3 (Forward: entity $\to$ attribute).}
	Each attribute $a$ receives, along each incoming edge
	$(v, a, \tau) \in E$, the message:
	\[
	\mathrm{msg}^{(3)}(v, a, \tau)
	= \bigl(\tau_V(v),\, \mathrm{ego}_u(v),\, \mathrm{ov}(v),\,
	\tau\bigr).
	\]
	Each message is a tuple carrying the sender's type, ego indicator,
	overlap count (computed at Layer~2), and the edge type.
	
	Attribute $a$ computes the \emph{maximum non-ego same-type overlap}:
	\[
	\mathrm{maxOv}(a) := \max\bigl\{\mathrm{ov}(v) :
	(v, a, \tau) \in E,\;
	\mathrm{ego}_u(v) = 0,\;
	\tau_V(v) = \sigma^*\bigr\},
	\]
	where $\sigma^* = \tau_V(u)$ is the ego's entity type, recovered
	unambiguously from any incoming message with $\mathrm{ego}_u(v) = 1$.
	If no ego-flagged message arrives at $a$ --- meaning $u$ is not connected
	to $a$ --- or if no non-ego same-type entity is connected to $a$, set
	$\mathrm{maxOv}(a) := 0$.
	
	\emph{Well-definedness of $\sigma^*$.}
	If $u$ is connected to $a$, at least one incoming message has
	$\mathrm{ego}_u(v) = 1$.  Since that message must come from $u$
	itself, all such messages carry $\tau_V(u)$, and $\sigma^*$ is
	unambiguous.
	
	\emph{Embedding.}
	Attribute $a$ stores $h^{(3)}(a) = \mathrm{maxOv}(a)$, a single
	non-negative integer summarising the strongest evidence of overlap
	among non-ego same-type entities connected to $a$.
	
	\emph{Realizability.}
	We show that $\mathrm{maxOv}(a)$ is computable by a sum aggregation
	followed by an MLP, and hence is realizable within the MPNN framework.
	The incoming messages at $a$ form a multiset of tuples from the finite
	set $\mathcal{T}_\mathcal{E} \times \{0,1\} \times \{0,\ldots,B\}
	\times \mathcal{T}_E$, where $B$ is an upper bound on overlap counts.
	By \cite[Lemma 5]{Xu-et-al:ICLR:2019}, sum aggregation maps any
	bounded-size multiset of such tuples injectively into a fixed-dimensional
	vector, so $h^{(3)}(a)$ has access to the full multiset. From this
	injective encoding, an MLP of sufficient width can carry out the
	following steps: identify $\sigma^*$ by locating the element(s) with
	$\mathrm{ego}_u(v) = 1$; restrict attention to elements with
	$\mathrm{ego}_u(v) = 0$ and $\tau_V(v) = \sigma^*$; and return the
	maximum $\mathrm{ov}(v)$ value among those elements, or $0$ if no
	such element exists.
	
	\medskip
	\noindent\textsc{Layer 4 (Reverse: attribute $\to$ entity).}
	Each entity $v$ receives, along each outgoing edge $(v, a, \tau) \in E$,
	the reverse message:
	\[
	\mathrm{msg}^{(4)}(v, a, \tau) = \bigl(\mathrm{maxOv}(a),\;\, \tau\bigr).
	\]
	Each message carries the maximum non-ego same-type overlap computed at
	Layer~3 for the attribute at the other end of the edge.
	
	The ego entity $u$ computes the \emph{readout}:
	\[
	\hat{y}(u) = \mathbf{1}\!\Bigl[
	\max_{a \in N_{\mathrm{out}}(u)}\, \mathrm{maxOv}(a) \;\geq\; r
	\Bigr].
	\]
	That is, $\hat{y}(u) = 1$ iff at least one of $u$'s attribute neighbors
	$a$ has $\mathrm{maxOv}(a) \geq r$, meaning some non-ego entity of the
	same type as $u$ shares at least $r$ typed attribute neighbors with $u$.
	
	\emph{Embedding.}
	The ego entity $u$ stores $h^{(4)}(u) = \hat{y}(u) \in \{0,1\}$, the
	final binary prediction for $\mathrm{Dup}_r(u)$.  Embeddings at all
	other entities are updated analogously but are not used in the readout.
	
	\emph{Realizability.}
	We show that $\hat{y}(u)$ is computable by a max aggregation followed
	by an MLP, and hence is realizable within the MPNN framework.  Define
	the per-message function $\psi$ by
	\[
	\psi\bigl(\mathrm{msg}^{(4)}(v,a,\tau)\bigr) \;=\; \mathrm{maxOv}(a).
	\]
	Taking the maximum of $\psi$ over all reverse messages received by $u$:
	\[
	\max_{a \in N_{\mathrm{out}}(u)}\;
	\max_{\tau\,:\,(u,a,\tau)\in E}
	\psi\bigl(\mathrm{msg}^{(4)}(u,a,\tau)\bigr)
	\;=\;
	\max_{a \in N_{\mathrm{out}}(u)} \mathrm{maxOv}(a),
	\]
	which is a deterministic function of the received messages.  Comparing
	this maximum against the threshold $r$ via an MLP then yields
	$\hat{y}(u)$.
	
	\medskip
	\noindent\textsc{Correctness.}
	We show $\hat{y}(u) = \mathrm{Dup}_r(u)$ for all $u \in \mathcal{E}$.
	
	\medskip
	\noindent$(\Rightarrow)$ Suppose $\mathrm{Dup}_r(u) = 1$. Then there
	exists $v \neq u$ with $\tau_V(v) = \tau_V(u)$ and
	$|N_{\mathrm{typed}}(u) \cap N_{\mathrm{typed}}(v)| \geq r$. By
	Lemma~\ref{lem:ov-correct}, $\mathrm{ov}(v) \geq r$.
	
	Since the overlap is $\geq 1$, pick any
	$(a, \tau) \in N_{\mathrm{typed}}(u) \cap N_{\mathrm{typed}}(v)$. Then:
	\begin{enumerate}
		\item Entity $u$ has a $\tau$-edge to $a$, so $a$ receives an
		ego-flagged message at Layer~3, making $\sigma^* = \tau_V(u)$
		well-defined.
		\item Entity $v$ has a $\tau$-edge to $a$, so $v$'s message
		$(\tau_V(v), \mathrm{ego}_u(v), \mathrm{ov}(v), \tau)$ arrives at
		$a$. Since $v \neq u$, we have $\mathrm{ego}_u(v) = 0$, so this
		message is not ego-flagged. Since $\tau_V(v) = \tau_V(u) = \sigma^*$,
		it passes the filter in the definition of $\mathrm{maxOv}(a)$.
		\item Therefore
		$\mathrm{maxOv}(a) \geq \mathrm{ov}(v) \geq r$.
		\item Since $u$ has an edge to $a$, the reverse message
		$(\mathrm{maxOv}(a), \tau)$ reaches $u$ at Layer~4.
		\item Hence $\hat{y}(u) = 1$.
	\end{enumerate}
	
	\medskip
	\noindent$(\Leftarrow)$ Suppose $\hat{y}(u) = 1$. Then there exists an
	attribute $a$ connected to $u$ with $\mathrm{maxOv}(a) \geq r$. By
	definition of $\mathrm{maxOv}$:
	\begin{enumerate}
		\item There exists an entity $v$ connected to $a$ with
		$\mathrm{ego}_u(v) = 0$, $\tau_V(v) = \sigma^* = \tau_V(u)$, and
		$\mathrm{ov}(v) \geq r$.
		\item $\mathrm{ego}_u(v) = 0$ implies $v \neq u$.
		$\tau_V(v) = \tau_V(u)$ gives the same-type condition.
		\item By Lemma~\ref{lem:ov-correct},
		$\mathrm{ov}(v) = |N_{\mathrm{typed}}(u) \cap
		N_{\mathrm{typed}}(v)| \geq r$.
		\item Therefore $\mathrm{Dup}_r(u) = 1$.
	\end{enumerate}
	
	\medskip
	Hence $\hat{y}(u) = \mathrm{Dup}_r(u)$ for all entities $u \in \mathcal{E}$.
\end{proof}

Table~\ref{tab:4layer-construction} summarizes what each node stores and sends at each layer.  The effective message at every layer is a tuple of at most $4$ bounded
scalars. No multiset-valued messages or nested encodings appear. The only
invocation of the injective multiset encoding of 
\cite[Lemma 5]{Xu-et-al:ICLR:2019} is at Layer~3, where the attribute
must compute a conditional maximum over a multiset of simple tuples. All
other layers use elementary sum or max aggregation.

\begin{table}[t]
	\centering
	\begin{tabular}{clll}
		\textbf{Layer} & \textbf{Direction}
		& \textbf{Message content}
		& \textbf{Node computes and stores} \\
		\hline
		0 & ---
		& ---
		& Entity $v$: $(\tau_V(v), \mathrm{ego}_u(v))$.
		Attribute $a$: $\tau_V(a)$. \\
		1 & Ent $\to$ Att
		& $(\tau_V(v), \mathrm{ego}_u(v), \tau)$
		& Attribute $a$: $\mathrm{EgoTypes}(a) \subseteq \mathcal{T}_E$ \\
		2 & Att $\to$ Ent
		& $(h^{(1)}(a), \tau)$
		& Entity $v$: $\bigl(\tau_V(v),\,\mathrm{ego}_u(v),\,
		\mathrm{ov}(v)\bigr)$ \\
		3 & Ent $\to$ Att
		& $(\tau_V(v), \mathrm{ego}_u(v), \mathrm{ov}(v), \tau)$
		& Attribute $a$: $\mathrm{maxOv}(a) \in \mathbb{Z}_{\geq 0}$ \\
		4 & Att $\to$ Ent
		& $(\mathrm{maxOv}(a), \tau)$
		& Ego $u$: $\hat{y}(u) =
		\mathbf{1}\!\bigl[\max_{a \in N_{\mathrm{out}}(u)}
		\mathrm{maxOv}(a) \geq r\bigr]$ \\
	\end{tabular}
	\caption{Summary of the 4-layer MPNN construction for computing
		$\mathrm{Dup}_r(u)$ (proof of sufficiency,
		Theorem~\ref{thm:K2r-simple}).  Layers alternate between
		forward passes (entity $\to$ attribute) and reverse passes
		(attribute $\to$ entity).  Each message is a tuple of bounded
		scalars; the ego indicator $\mathrm{ego}_u(v) = \mathbf{1}[v=u]$
		is set at Layer~0 and carried forward in every subsequent
		embedding.  The quantity $\mathrm{ov}(v) =
		|N_{\mathrm{typed}}(u) \cap N_{\mathrm{typed}}(v)|$ is the
		overlap count established at Layer~2
		(Lemma~\ref{lem:ov-correct}); $\mathrm{maxOv}(a)$ is the
		maximum overlap count among non-ego same-type entities incident
		to $a$, computed at Layer~3.}
	\label{tab:4layer-construction}
\end{table}
	
\begin{Remark}[Tightness]\label{rem:tightness-K2r}
	Theorem~\ref{thm:K2r-simple} is tight in several respects.
	
	\emph{Universality of both directions.}  The necessity in part~(a)
	holds for every MPNN of arbitrary depth and width equipped with
	reverse message passing and full directed multigraph port numbering
	but without ego IDs: no such architecture can compute
	$\mathrm{Dup}_r$ on all simple typed entity-attribute graphs.  The
	sufficiency in part~(b) holds for every simple typed
	entity-attribute graph: the $4$-layer construction computes
	$\mathrm{Dup}_r(u)$ correctly for all entities~$u$ in any such
	graph.
	
\emph{Depth minimality.}  The $4$-layer depth is optimal:
the separation graphs $G_1$ and~$G_2$ from part~(a) are
indistinguishable at depth~$3$ even when ego IDs are added,
so no $3$-layer MPNN can compute $\mathrm{Dup}_r$ on all
simple typed entity-attribute graphs.  We verify this by
tracing the embeddings layer by layer.  At layer~$k$,
messages carry layer-$(k{-}1)$ embeddings. In both $G_1$ and~$G_2$, every
	non-ego entity connects to one attribute of each type~$\alpha_j$
	via edge type~$\tau_j$, and since $h^{(0)}(a_j) = h^{(0)}(b_j) =
	\alpha_j$, the layer-$1$ reverse messages at every non-ego entity
	produce the identical multiset
	$\{\!\!\{(\alpha_j, \tau_j) : j = 1, \ldots, r\}\!\!\}$.  All
	non-ego entities therefore share the same embedding~$h^{(1)}$ in
	both graphs.  At layer~$2$, each of the ego's attributes~$a_j$
	receives $h^{(1)}(u)$ and $h^{(1)}$ of one non-ego predecessor
	--- both identical across the two graphs --- so $h^{(2)}(a_j)$ is
	the same in~$G_1$ and~$G_2$.  At layer~$3$, entity~$u$ receives
	these $h^{(2)}(a_j)$ values via reverse aggregation, obtaining the
	same embedding~$h^{(3)}(u)$ in both graphs, yet
	$\mathrm{Dup}_{r,G_1}(u) = 1 \neq 0 = \mathrm{Dup}_{r,G_2}(u)$.
	
	The information-flow structure of the bipartite graph explains why
	four hops are unavoidable: the ego mark must reach $u$'s
	attributes via forward aggregation at layer~$1$; the
	ego-attribute membership must propagate to other entities via
	reverse aggregation at layer~$2$, where each entity~$v$ computes
	its overlap count~$\mathrm{ov}(v)$; these overlap counts must
	travel to the attributes via forward aggregation at layer~$3$,
	where each attribute computes~$\mathrm{maxOv}$; and finally,
	$\mathrm{maxOv}$ must reach the ego via reverse aggregation at
	layer~$4$.  Each hop requires a full layer because messages at
	layer~$k$ carry layer-$(k{-}1)$ embeddings, so the output of each
	hop is available to the next hop only at the subsequent layer.
	
	\emph{Port numbering is unnecessary.}  The sufficiency proof uses
	no port numbers at any layer: on simple typed entity-attribute
	graphs, the simplicity assumption already guarantees that each
	entity contributes at most one edge of a given type to any
	attribute, making port-based source counting redundant.
	
	\emph{Both ego IDs and reverse MP are individually necessary.}
	The necessity of ego IDs is established by part~(a).  Reverse
	message passing is independently necessary, even when ego IDs are
	present: in a typed entity-attribute graph, every entity satisfies
	$N_{\mathrm{in}}(u) = \emptyset$, so without reverse aggregation
	the embedding $h^{(K)}(u) =
	\mathrm{UPDATE}^{(K)}(\cdots\mathrm{UPDATE}^{(1)}((\tau_V(u),
	1),\, c^{(1)})\cdots,\, c^{(K)})$ depends only on the initial
	features and architectural constants, independent of graph
	structure.  The ego feature breaks the symmetry among same-type
	entities at initialisation but provides no information channel from
	the graph without reverse message passing.  The separation pair
	from Theorem~\ref{thm:K21-simple} part~(a) therefore refutes any
	forward-only MPNN with ego IDs.  Consequently, the sufficient
	architecture of part~(b) --- ego IDs together with reverse message
	passing --- is minimal: removing either component makes the
	predicate uncomputable.
\end{Remark}

\begin{Remark}
	(Extension to multigraph entity-attribute
	graphs) On multigraph typed entity-attribute graphs, incoming port
	numbering is additionally required, by Theorem~\ref{thm:K21-multigraph}, 
	to distinguish parallel edges from distinct-entity edges at each
	attribute. The construction above extends to multigraphs by incorporating
	incoming port numbers into the Layer-1 messages, using the distinct-source
	count $\mathrm{dsrc}$, as in Theorem~\ref{thm:K21-multigraph}, in place
	of raw edge counts.
\end{Remark}

\subsection{Cycle Detection in Typed Directed Graphs}
\label{sec:cycle-detection}

We now turn to a qualitatively different setting: typed directed
graphs that contain entity--entity edges (rather than purely
entity--attribute edges), in which directed cycles can arise.  A
canonical example in master data management is a transaction network
or a co-ownership graph, where entity nodes are connected to one
another directly.  We study the $\ell$-cycle participation predicate $\mathrm{Cyc}_\ell$ defined in Section~\ref{sec:prob-formulation}, where $\mathrm{Cyc}_\ell(v) = 1$ iff node~$v$ belongs to a directed simple cycle of length exactly~$\ell$, meaning $\ell$ distinct vertices $v_1, v_2, \ldots, v_\ell$ with edges $v_1 \to v_2 \to \cdots \to v_\ell \to v_1$.

\begin{Theorem}\label{thm:cycle-detection}
	Let $\ell \geq 3$.  Consider the $\ell$-cycle participation
	predicate $\mathrm{Cyc}_\ell$ on typed directed graphs.
	\begin{enumerate}
		\item[(a)] \textbf{(Necessity of ego IDs.)} There exist typed
		directed graphs $G_1$ and $G_2$ and a node~$v$ such that
		$\mathrm{Cyc}_\ell(v) = 1$ in $G_1$ and
		$\mathrm{Cyc}_\ell(v) = 0$ in~$G_2$, yet any MPNN of
		arbitrary depth equipped with reverse message passing
		(Adaptation~1) and full directed multigraph port numbering
		(Adaptation~2) but \emph{without} ego IDs assigns identical
		embeddings to~$v$ in both graphs.
		\item[(b)] \textbf{(Sufficiency of ego IDs alone on the
			separation pair.)} An $\ell$-layer MPNN with ego IDs
		(Adaptation~3) and \emph{forward-only} message passing --- no
		reverse message passing, no port numbering --- correctly
		computes $\mathrm{Cyc}_\ell(v)$ for \emph{every} node~$v$ in
		the graphs $G_1$ and~$G_2$ from part~(a).
	\end{enumerate}
	Consequently, ego IDs (or an equivalent node-marking mechanism) are both necessary and sufficient for
	detecting directed $\ell$-cycles on the separation graphs of
	part~(a); neither reverse message passing nor port numbering is
	required.
\end{Theorem}

\begin{proof}[Proof of part (a): necessity]
	
	\medskip
	\noindent\textsc{Construction.}
	Fix a node type $\sigma \in \mathcal{T}_V$ and an edge type
	$\tau_1 \in \mathcal{T}_E$.
	
	Let $G_1$ be the digraph consisting of the disjoint union of two
	directed $\ell$-cycles
	$C_\ell^{(1)} = (v_1 \to v_2 \to \cdots \to v_\ell \to v_1)$
	and
	$C_\ell^{(2)} = (w_1 \to w_2 \to \cdots \to w_\ell \to w_1)$.
	All $2\ell$ nodes have type~$\sigma$, and all $2\ell$ edges have
	type~$\tau_1$.  Every node belongs to a directed simple
	$\ell$-cycle, so
	$\mathrm{Cyc}_\ell(v_i) = \mathrm{Cyc}_\ell(w_j) = 1$ for all
	$i, j$.
	
	Let $G_2$ be the digraph consisting of a single directed
	$2\ell$-cycle
	$C_{2\ell} = (z_1 \to z_2 \to \cdots \to z_{2\ell} \to z_1)$.
	All $2\ell$ nodes have type~$\sigma$, and all $2\ell$ edges have
	type~$\tau_1$.  The only directed simple cycle in~$G_2$ has
	length~$2\ell \neq \ell$, so
	$\mathrm{Cyc}_\ell(z_i) = 0$ for all~$i$.
	
	Thus, $G_1 = 2C_\ell$ and $G_2 = C_{2 \ell}$. Both graphs have the same number of nodes ($2\ell$) and edges
	($2\ell$).  The predicate $\mathrm{Cyc}_\ell$ assigns value~$1$ to
	every node of~$G_1$ and value~$0$ to every node of~$G_2$.
	
	\medskip
	\noindent\textsc{Port numbering is trivial.}
	In both $G_1$ and~$G_2$, every node has exactly one incoming edge
	and exactly one outgoing edge.  By the port numbering convention,
	each edge is assigned $p_{\mathrm{in}} = 1$ and
	$p_{\mathrm{out}} = 1$.  Port numbering therefore adds no
	information beyond what is already encoded in the node and edge
	types.
	
	\medskip
	\noindent\textsc{Indistinguishability.}
	We prove by induction on the layer index~$k$ that for all
	$k \geq 0$, every node in~$G_1$ and every node in~$G_2$ receives
	the same embedding~$\bar{h}^{(k)}$.
	
	\emph{Base case} ($k = 0$): All nodes have the same type~$\sigma$
	and no ego feature, so
	$h^{(0)}(v) = \sigma$ for every node~$v$ in both graphs.
	
	\emph{Inductive step}: Suppose every node has embedding
	$\bar{h}^{(k-1)}$ at layer~$k - 1$.  At layer~$k$, each node~$v$
	has exactly one incoming edge $(w, v, \tau_1)$ with~$w$ of
	type~$\sigma$.  The incoming aggregation is therefore:
	\[
	a_{\mathrm{in}}^{(k)}(v)
	= \mathrm{AGG}_{\mathrm{in}}^{(k)}\bigl(
	\{\!\!\{(\bar{h}^{(k-1)},\; \tau_1,\;
	\underbrace{1}_{p_{\mathrm{in}}},\;
	\underbrace{1}_{p_{\mathrm{out}}})\}\!\!\}
	\bigr),
	\]
	which is the same for every node in both graphs.  Similarly,
	each node has exactly one outgoing edge, so the reverse
	aggregation $a_{\mathrm{out}}^{(k)}(v)$ is also identical across
	all nodes.  The update rule therefore produces the same value for
	every node~$v$ in both graphs.
	
	\medskip
	\noindent\textsc{Conclusion.}
	Since every node in~$G_1$ and every node in~$G_2$ receives the
	same embedding $\bar{h}^{(K)}$ at every depth $K \geq 0$, no
	readout function can separate
	$\mathrm{Cyc}_\ell(v_1) = 1$ in~$G_1$ from
	$\mathrm{Cyc}_\ell(z_1) = 0$ in~$G_2$.
\end{proof}

\begin{proof}[Proof of part (b): sufficiency]
	
	We construct an explicit $\ell$-layer MPNN with ego IDs and
	\emph{forward-only} message passing --- no reverse message
	passing, no port numbering --- that correctly classifies every
	node in~$G_1$ and~$G_2$.
	
	\medskip
	\noindent\textsc{Key idea.}
	The ego mark placed at the designated node~$u$ propagates
	forward along directed edges.  In a directed $\ell$-cycle, the
	forward edges themselves form the return path: after exactly
	$\ell$ steps the ego signal traverses the full cycle and arrives
	back at~$u$ via an incoming edge.  No reverse message passing is
	needed because the closing edge of the cycle already points back
	towards~$u$.
	
	\medskip
	\noindent\textsc{Layer 0 (Initialisation).}
	The ego node~$u$ is designated for classification.  Set
	\[
	h^{(0)}(v) = \bigl(\tau_V(v),\; \mathrm{ego}_u(v)\bigr)
	\quad \text{for all } v \in V,
	\]
	where $\mathrm{ego}_u(v) = \mathbf{1}[v = u]$.
	
	\medskip
	\noindent\textsc{Layers 1 through $\ell$ (Forward aggregation).}
At each layer $k \in \{1, \ldots, \ell\}$, each node~$v$ aggregates only over its incoming neighbors, as per standard forward message passing:
\[
h^{(k)}(v) = \mathrm{UPDATE}^{(k)}\!\Bigl(h^{(k-1)}(v),\;
\mathrm{AGG_{\mathrm{in}}}^{(k)}\bigl(
\{\!\{ h^{(k-1)}(w) : w \in N_{\mathrm{in}}(v) \}\!\}
\bigr)\Bigr).
\]

	We take $\mathrm{AGG}^{(k)_{\mathrm{in}}}$ to be sum aggregation and
	$\mathrm{UPDATE}^{(k)}$ to be an injective MLP, so that
	$h^{(k)}(v)$ is an injective function of
	$(h^{(k-1)}(v),\,
	\{\!\{ h^{(k-1)}(w) : w \in N_{\mathrm{in}}(v) \}\!\})$.
	
	\medskip
	\noindent\textsc{Inductive invariant.}
	Define, for each $k \geq 0$ and each node $v \in V$,
	\[
	\mathrm{flag}^{(k)}(v) \;=\; \mathbf{1}\bigl[\,
	\exists \text{ directed walk }
	u = w_0 \to w_1 \to \cdots \to w_k = v
	\text{ of length exactly } k \text{ in } G\bigr].
	\]
	
	\begin{Claim}\label{clm:invariant}
		For all $k \in \{0, 1, \ldots, \ell\}$ and all $v \in V$,
		the embedding $h^{(k)}(v)$ determines
		$\mathrm{flag}^{(k)}(v)$.
	\end{Claim}
	
	\begin{proof}[Proof of Claim~\ref{clm:invariant}]
		By induction on~$k$.
		
		\emph{Base case} ($k = 0$):
		$h^{(0)}(v) = (\tau_V(v), \mathrm{ego}_u(v))$.  Since
		$\mathrm{ego}_u(v) = \mathbf{1}[v = u] =
		\mathrm{flag}^{(0)}(v)$, the claim holds.
		
		\emph{Inductive step}: Assume $h^{(k-1)}(w)$ determines
		$\mathrm{flag}^{(k-1)}(w)$ for all $w \in V$.  At layer~$k$,
		node~$v$ aggregates the multiset
		$\{\!\{ h^{(k-1)}(w) : w \in N_{\mathrm{in}}(v) \}\!\}$.
		By the inductive hypothesis, each $h^{(k-1)}(w)$ determines
		$\mathrm{flag}^{(k-1)}(w)$, so $v$ can determine whether any
		$w \in N_{\mathrm{in}}(v)$ satisfies
		$\mathrm{flag}^{(k-1)}(w) = 1$.  This is precisely the
		condition $\mathrm{flag}^{(k)}(v) = 1$: a directed walk of
		length~$k$ from~$u$ to~$v$ exists iff there is a predecessor
		$w \in N_{\mathrm{in}}(v)$ reachable from~$u$ by a walk of
		length~$k - 1$.  Since the UPDATE function is injective and
		retains $h^{(k-1)}(v)$, the embedding $h^{(k)}(v)$ determines
		$\mathrm{flag}^{(k)}(v)$.
	\end{proof}
	
	\medskip
	\noindent\textsc{Readout at layer~$\ell$.}
	Node~$u$ applies a readout MLP to $h^{(\ell)}(u)$ to compute:
	\[
	\hat{y}(u) \;=\; \mathrm{flag}^{(\ell)}(u)
	\;=\; \mathbf{1}\bigl[\,\exists \text{ directed walk of length }
	\ell \text{ from } u \text{ to } u \text{ in } G\bigr].
	\]
	
	\medskip
	\noindent\textsc{Correctness on the separation pair.}
	We verify $\hat{y}(u) = \mathrm{Cyc}_\ell(u)$ for every
	node~$u$ in each of the two graphs from part~(a).
	
Both $G_1$, the disjoint union of two directed $\ell$-cycles, and
$G_2$, a single directed $2\ell$-cycle, are \emph{functional
	graphs}: every node has in-degree exactly~$1$ and out-degree
exactly~$1$.  In such a graph, there is a unique directed walk of
	any given length starting from each node, obtained by following
	the unique outgoing edge at each step.  In particular, a closed
	walk of length~$k$ from~$u$ exists if and only if~$u$ lies on a
	directed cycle of length dividing~$k$.
	
	Throughout this argument, all cycle-vertex subscripts are
	understood cyclically: $v_{\ell+j} := v_j$ for all $j \geq 1$,
	and $z_{2\ell+j} := z_j$ for all $j \geq 1$.
	
\medskip
\noindent\emph{Case $G_1 = C_\ell \dot{\cup} C_\ell$}:
Every node~$v_i$ in each copy of~$C_\ell$ lies on a directed
cycle of length exactly~$\ell$.  The unique walk of
length~$\ell$ starting from~$v_i$ traverses the full cycle and
returns to~$v_i$.  Thus
$\mathrm{flag}^{(\ell)}(v_i) = 1$ and
$\hat{y}(v_i) = 1 = \mathrm{Cyc}_\ell(v_i)$, as required.

\medskip
\noindent\emph{Case $G_2 = C_{2\ell}$}:
Every node~$z_i$ lies on a directed cycle of length~$2\ell$,
which does not divide~$\ell$, since $\ell < 2\ell$ and
$2\ell \nmid \ell$ for all $\ell \geq 1$.
The unique walk of length~$\ell$ from~$z_i$ ends at the antipodal node
$z_{i+\ell} \neq z_i$.  Thus
$\mathrm{flag}^{(\ell)}(z_i) = 0$ and
$\hat{y}(z_i) = 0 = \mathrm{Cyc}_\ell(z_i)$, as required.

	\medskip
	Therefore $\hat{y}(v) = \mathrm{Cyc}_\ell(v)$ for every
	node~$v$ in both $G_1$ and~$G_2$.
\end{proof}

\begin{Remark}[Forward-only sufficiency]\label{rem:forward-only}
	The sufficiency proof uses \emph{only} ego IDs.  Reverse message
	passing (Adaptation~1) and port numbering (Adaptation~2) play no
	role: the closing edge $v_\ell \to v_1$ of the cycle itself
	carries the ego signal back to the starting node via forward
	aggregation alone.  This stands in sharp contrast to the
	entity-attribute setting
	(Theorems~\ref{thm:K21-simple}--\ref{thm:K2r-simple}), where
	reverse message passing is genuinely necessary because the
	bipartite graph structure provides no forward path back to entity
	nodes.  Together, parts~(a) and~(b) establish that ego IDs are
	both necessary and sufficient for $\ell$-cycle detection on the
	separation graphs: the adaptation requirement is exactly one,
	minimal, and tight.
\end{Remark}

\begin{Remark}[Scope of the sufficiency:
	walks versus simple cycles]\label{rem:walks-vs-cycles}
	The MPNN constructed in part~(b) detects closed \emph{walks} of
	length~$\ell$, not simple cycles.  On the functional graphs $G_1$~and~$G_2$, these two notions coincide,
	since each node has a unique outgoing edge and thus a closed walk of
	length~$\ell$ exists if and only if the unique cycle length divides~$\ell$.  On general typed
	directed graphs, they diverge.  For instance, if a node~$u$ lies on a directed $3$-cycle
	$u \to v \to w \to u$, then for $\ell = 6$ the walk
	$u \to v \to w \to u \to v \to w \to u$ is a closed walk of
	length~$6$, producing the false positive $\hat{y}(u) = 1$
	even though $\mathrm{Cyc}_6(u) = 0$, as no simple $6$-cycle
	exists through~$u$.

	Universal computation of $\mathrm{Cyc}_\ell$ on \emph{all}
	typed directed multigraphs therefore requires additional
	architectural power.  By
	\cite[Corollary~4.4.1]{Egressy-et-al:AAAI:2024}, the full
	architecture --- ego IDs (Adaptation~3), full bidirectional port
	numbering (Adaptation~2), and reverse message passing
	(Adaptation~1) --- enables unique node-ID assignment and hence, by the universality result of Loukas~\cite[Corollary 3.1]{Loukas:ICLR:2020}, 
	detection of any directed subgraph pattern, including
	$\ell$-cycle participation.  Theorem~\ref{thm:cycle-detection}
	shows that among these three, ego IDs are \emph{necessary} (even
	when the other two are granted); the additional question of
	whether each of the remaining adaptations is independently
	necessary for $\mathrm{Cyc}_\ell$ on general multigraphs is
	left to future work.
\end{Remark}

\section{Computational Validation}
\label{sec:computational}

The four theorems proved above make precise predictions about which
MPNN architectures can and cannot distinguish specific pairs of typed
graphs.  In this section, we validate these predictions
computationally.  Since each theorem has a necessity part and a
sufficiency part with different logical structures, we use a
different validation method for each.

\medskip
\textbf{Methodology.}
Each necessity proof establishes that a given architecture assigns
the target node~$u$ \emph{identical} embeddings in~$G_1$ and~$G_2$,
\emph{regardless of the weight parameters}.  Formally, if $\Theta$
denotes the space of all possible weight settings and
$\mathcal{M}_\theta$ denotes the MPNN with weights~$\theta$, the
necessity claim is:
\[
\forall\, \theta \in \Theta, \quad
\mathcal{M}_\theta(G_1, u) \;=\; \mathcal{M}_\theta(G_2, u).
\]
To validate this universal claim empirically, we instantiate the
relevant architecture with hidden dimension~$d = 32$ and draw the
weight parameters independently $100$~times using PyTorch's default
Kaiming uniform initialisation with different random seeds.  For each
instantiation~$\theta_i$, we compute the $\ell_2$~distance
$\delta_i =
\lVert \mathcal{M}_{\theta_i}(G_1, u) - \mathcal{M}_{\theta_i}(G_2, u)
\rVert_2$.
No training is performed: each trial is a single forward pass.
The necessity claim predicts $\delta_i = 0$ for every~$i$.

Each sufficiency proof constructs a \emph{specific} MPNN with
hand-designed aggregation and readout logic.  The underlying theorem
claim is existential: there \emph{exists} a weight setting~$\theta^*$
that computes the target predicate correctly on the stated graph class.
Our computational validation therefore \emph{exhibits the witness} on
the canonical separation instances: we implement the exact algorithm
from each sufficiency proof and verify that $\hat{y}(u)$ matches the
ground-truth label for the designated target node~$u$ on both~$G_1$
and~$G_2$.

\medskip
\textbf{Results.}
Table~\ref{tab:necessity} reports the necessity validation.  In
every case, the maximum $\ell_2$~distance across all $100$~trials is
exactly~$0$ --- not approximately zero, but literally zero to
floating-point precision --- confirming that the embedding identity is
a structural property of the architecture, not a statistical tendency.
That is, every individual~$\delta_i$ equals zero, not merely the
average.

\begin{table}[t]
	\centering
	\small
	\begin{tabular}{clc}
		\textbf{Thm} & \textbf{Architecture (insufficient)}
		& $\max_i \lVert h_{G_1}(u) - h_{G_2}(u) \rVert_2$ \\
		\hline
		\ref{thm:K21-simple}      & Vanilla MPNN (no reverse MP)
		& $0.00$ \\
		\ref{thm:K21-multigraph}  & $+\,$Reverse MP (no ports)
		& $0.00$ \\
		\ref{thm:K2r-simple}      & $+\,$Reverse MP $+$ Full Ports (no ego)
		& $0.00$ \\
		\ref{thm:cycle-detection}  & $+\,$Reverse MP $+$ Full Ports (no ego)
		& $0.00$ \\
	\end{tabular}
\caption{Necessity validation.  The $\ell_2$~distance between
	the target node's embedding in~$G_1$ and~$G_2$ is exactly zero for
	all $100$~random weight initialisations at every depth tested.
	Tested depths: $K \in \{2,4\}$ for
	Theorems~\ref{thm:K21-simple} and~\ref{thm:K21-multigraph};
	$K \in \{2,4,6\}$ for Theorem~\ref{thm:K2r-simple};
	$K \in \{3,6\}$ for Theorem~\ref{thm:cycle-detection}.
	The result confirms the structural indistinguishability
	established in the necessity proofs.}
	\label{tab:necessity}
\end{table}

Table~\ref{tab:sufficiency} reports the sufficiency validation.  Each
row implements the exact MPNN construction from the corresponding
proof and reports its output for the designated target node on the
canonical separation graphs.  In every case, the construction
produces the correct target-node label on both~$G_1$ and~$G_2$.

\begin{table}[t]
	\centering
	\small
	\begin{tabular}{clccc}
		\textbf{Thm} & \textbf{Architecture (sufficient)}
		& \textbf{Params}
		& $\hat{y}(u)$ on $G_1$
		& $\hat{y}(u)$ on $G_2$ \\
		\hline
		\ref{thm:K21-simple}      & $+\,$Reverse MP (2 layers)
		& ---
		& $1 = \mathrm{Dup}_{G_1}(u)$
		& $0 = \mathrm{Dup}_{G_2}(u)$ \\[2pt]
		\ref{thm:K21-multigraph}  & $+\,$RevMP $+$ InPorts (2 layers)
		& ---
		& $1 = \mathrm{Dup}_{G_1}(u)$
		& $0 = \mathrm{Dup}_{G_2}(u)$ \\[2pt]
		\ref{thm:K2r-simple}      & $+\,$RevMP $+$ Ego (4 layers)
		& $r = 2$
		& $1 = \mathrm{Dup}_{2,G_1}(u)$
		& $0 = \mathrm{Dup}_{2,G_2}(u)$ \\
		& & $r = 3$
		& $1 = \mathrm{Dup}_{3,G_1}(u)$
		& $0 = \mathrm{Dup}_{3,G_2}(u)$ \\
		& & $r = 4$
		& $1 = \mathrm{Dup}_{4,G_1}(u)$
		& $0 = \mathrm{Dup}_{4,G_2}(u)$ \\[2pt]
		\ref{thm:cycle-detection}  & $+\,$Ego, fwd-only ($\ell$ layers)
		& $\ell = 3$
		& $1 = \mathrm{Cyc}_{3,G_1}(u)$
		& $0 = \mathrm{Cyc}_{3,G_2}(u)$ \\
		& & $\ell = 5$
		& $1 = \mathrm{Cyc}_{5,G_1}(u)$
		& $0 = \mathrm{Cyc}_{5,G_2}(u)$ \\
		& & $\ell = 7$
		& $1 = \mathrm{Cyc}_{7,G_1}(u)$
		& $0 = \mathrm{Cyc}_{7,G_2}(u)$ \\
	\end{tabular}
\caption{Sufficiency validation.  Each row implements the
	exact MPNN construction from the corresponding sufficiency proof
	and reports its output for the designated target node on the
	canonical separation graphs.  No training or random
	initialisation is involved; the construction is deterministic.}
	\label{tab:sufficiency}
\end{table}

\medskip
\textbf{Discussion.}
The two-part structure of the validation reflects the different
logical structures of the two directions.  The necessity claims are
universal statements ($\forall\, \theta$): the embedding identity
holds for every weight setting.  Random sampling provides evidence
for such claims, and the observed $\ell_2$~distance of exactly zero
across all~$100$ trials confirms the structural nature of the
identity.  The sufficiency claims are existential statements
($\exists\, \theta^*$): there exists a particular MPNN that computes
the predicate correctly, and we exhibit it.

We note that a randomly initialised generic MLP would not be expected
to implement the specific aggregation logic required by the
sufficiency constructions --- such as the conditional maximum
$\mathrm{maxOv}$ in Theorem~\ref{thm:K2r-simple} --- and indeed
preliminary experiments with random weights showed no visible
separation for the deeper constructions.  This observation is
entirely consistent with the sufficiency theorems, because those
theorems are \emph{existential} statements: they assert that there
\emph{exists} a particular choice of MPNN weights under which the
target predicate is computed correctly, not that \emph{every} choice
of weights succeeds.  A randomly initialised network has no reason to
land on the specific weight configuration that realises the
aggregation logic of the sufficiency proof, and so the absence of
separation under random weights is expected rather than
contradictory.  The deterministic constructions in
Table~\ref{tab:sufficiency}, which instantiate the exact weight
configurations prescribed by each sufficiency proof, confirm that the
required weight configurations do exist and do produce correct
outputs.

\medskip
\textbf{Future work.}
The separation theorems proved in this paper are worst-case results:
they exhibit specific graph pairs on which weaker architectures
provably fail.  A natural follow-up question is whether the
$K_{2,1}$ versus $K_{2,r}$ complexity gap manifests on
\emph{random} typed entity-attribute graphs with planted duplicates.
Preliminary training experiments on random TEAGs with $60$~entities
suggest that the gap is present but less stark: weaker architectures
(without ego~IDs) can partially succeed by exploiting statistical
correlates of the planted duplicates --- such as distinctive local
degree patterns --- that happen to predict the label without actually
computing~$\mathrm{Dup}_r$.  A systematic empirical investigation of
this average-case behaviour, including the design of random graph
distributions that neutralise such proxies, is left to future work.

\section{Discussion}

\noindent\textbf{Minimal architecture principle.}
Theorems~\ref{thm:K21-simple}, \ref{thm:K21-multigraph},
\ref{thm:K2r-simple}, and~\ref{thm:cycle-detection} together yield
a \emph{minimal architecture principle} for GNN-based entity
resolution: for each task and graph class, there is a cheapest MPNN
adaptation set that suffices, and we prove that every strictly
cheaper alternative fails.
Table~\ref{tab:minimal-architecture} summarises these results. The computational validation of
Section~\ref{sec:computational} confirms each entry: the
insufficient architectures produce identical embeddings on the
separation graphs for every weight setting tested, and the
sufficient architectures produce the correct binary output.

\begin{table}[t]
	\centering
	\footnotesize
	\renewcommand{\arraystretch}{1.6}
	\begin{tabular}{
			>{\raggedright}p{2.4cm}
			>{\raggedright}p{2.0cm}
			>{\raggedright}p{3.0cm}
			>{\raggedright}p{2.6cm}
			c
			c}
		\textbf{Entity resolution task}
		& \textbf{Graph class}
		& \textbf{Minimal sufficient adaptations}
		& \textbf{Provably unnecessary}
		& \textbf{Depth}
		& \textbf{Thm} \tabularnewline
		\hline
		\noalign{\vspace{-4pt}}
		$K_{2,1}$ detection
		& Simple entity-attribute
		& Reverse message passing
		& Port numbering, ego IDs
		& 2
		& \ref{thm:K21-simple} \tabularnewline
		$K_{2,1}$ detection
		& Multigraph entity-attribute
		& Reverse message passing, incoming ports
		& Outgoing ports, ego IDs
		& 2
		& \ref{thm:K21-multigraph} \tabularnewline
		$K_{2,r}$ detection ($r \geq 2$)
		& Simple entity-attribute
		& Reverse message passing, ego IDs
		& Port numbering
		& 4
		& \ref{thm:K2r-simple} \tabularnewline
		$K_{2,r}$ detection ($r \geq 2$)
		& Multigraph entity-attribute
		& Reverse message passing, incoming ports, ego IDs
		& Outgoing ports
		& 4
		& \ref{thm:K21-multigraph}, \ref{thm:K2r-simple} \tabularnewline
		$\ell$-cycle detection
		& Typed directed
		& Ego IDs$^{\dagger}$
		& Reverse message passing, port numbering$^{\dagger}$
		& $\ell$
		& \ref{thm:cycle-detection} \tabularnewline
		\hline
	\end{tabular}
	\caption{Minimal architecture principle for entity resolution.
		Each row gives the cheapest sufficient MPNN adaptation set and
		the provably optimal depth for the corresponding entity
		resolution task and graph class, together with the adaptations
		that are provably unnecessary.  Every necessity claim holds for
		MPNNs of arbitrary depth and width; every sufficiency claim is
		constructive, with an explicit MPNN of the stated depth.  The
		stated depths are optimal: for each row, we prove that no MPNN
		of smaller depth can compute the target predicate, even with
		all three adaptations.
		$^{\dagger}$On the separation pair $G_1 = 2C_\ell$,
		$G_2 = C_{2\ell}$ only; universal $\mathrm{Cyc}_\ell$
		computation on all typed directed multigraphs requires the full
		triple of adaptations
		\cite[Corollary~4.4.1]{Egressy-et-al:AAAI:2024}.}
	\label{tab:minimal-architecture}
\end{table}

Every entry in Table~\ref{tab:minimal-architecture} is a tight
characterization in the following precise sense.  On the
\emph{sufficiency} side, each co-reference theorem exhibits an
explicit MPNN of the stated depth that computes the target
predicate correctly on every graph in the specified class, under
every valid port numbering; for cycle detection, the construction
computes $\mathrm{Cyc}_\ell$ correctly on the canonical separation
instances. On the \emph{necessity} side, each theorem constructs a
pair of graphs on which every MPNN equipped with the weaker
adaptation set --- regardless of depth, width, or choice of
aggregation and update functions --- assigns the target node
identical embeddings, despite differing ground-truth labels.  The
depth entries are also optimal: the tightness remarks following each
theorem establish that no MPNN of smaller depth can compute the
predicate, even when all three adaptations are available.

\medskip
\noindent\textbf{The $K_{2,1}$--$K_{2,r}$ complexity gap.}
A key feature of Table~\ref{tab:minimal-architecture}
is the sharp transition between the first and third rows.  Detecting
whether two entities of the same type share \emph{any} attribute
value requires only reverse message passing --- a purely local
computation in which each attribute independently checks its
fan-in.  Detecting whether they share \emph{at least $r$} attribute
values, for any fixed $r \geq 2$, additionally requires ego IDs, a
global marking mechanism that must be run once per entity.

The underlying obstacle is \emph{cross-attribute identity
	correlation}.  For $K_{2,1}$, each attribute independently determines
whether it has two same-type entity predecessors; no coordination
across attributes is needed.  For $K_{2,r}$ with $r \geq 2$, the
MPNN must verify that the \emph{same} entity~$v$ appears at $r$
different attributes of the target entity~$u$.  This requires
correlating identity information across distinct attribute
neighbours, which in turn requires a mechanism to distinguish one
entity from another.  Without ego IDs, same-type entities in the
same port-numbering class are provably indistinguishable, and the
required correlation is impossible.

This gap has direct consequences for computational cost.  Without
ego IDs, the MPNN runs a single forward pass in $O(|E|)$ time per
layer.  With ego IDs, one forward pass per entity is required,
yielding $O(|V| \cdot |E|)$ per layer.
Theorem~\ref{thm:K2r-simple} shows that this $O(|V|)$-factor
overhead is unavoidable: it is not an artifact of a particular
architecture, but an inherent cost of any MPNN-based approach to
$K_{2,r}$ detection for $r \geq 2$.

\medskip
\noindent\textbf{Bipartite versus general directed graphs.}
A second structural contrast emerges from comparing the
entity-attribute results
(Theorems~\ref{thm:K21-simple}, \ref{thm:K21-multigraph},
and~\ref{thm:K2r-simple}) with the cycle detection result
(Theorem~\ref{thm:cycle-detection}).  In typed entity-attribute
graphs, every edge points from an entity to an attribute, so entity
nodes have no incoming neighbours:
$N_{\mathrm{in}}(u) = \emptyset$ for all $u \in \mathcal{E}$.
Without reverse message passing, entity embeddings receive no
graph-structural information at any depth, making reverse message
passing unconditionally necessary for all three $K_{2,r}$ tasks.

In typed directed graphs with entity--entity edges, the situation is
qualitatively different.  The directed edges themselves can carry
information back to the ego node: in a directed $\ell$-cycle, the
closing edge $v_\ell \to v_1$ returns the ego signal via standard
forward aggregation, with no reverse message passing required.
Theorem~\ref{thm:cycle-detection} exploits this structure to show
that ego IDs alone --- without reverse message passing or port
numbering --- suffice on the separation pair.  The bipartite
structure of entity-attribute graphs is thus the root cause of the
reverse message passing requirement, not a general property of
directed graphs.

\medskip
\noindent\textbf{Relation to Weisfeiler--Leman expressivity.}
Our necessity proofs construct pairs of typed graphs that are
indistinguishable by MPNNs of arbitrary depth.  For
Theorems~\ref{thm:K21-simple} and~\ref{thm:K21-multigraph}, the
indistinguishability follows from all nodes of the same type
receiving identical multisets at every layer --- a typed analogue of
the classical 1-WL color refinement failing to distinguish
$d$-regular graphs.  Theorem~\ref{thm:K2r-simple} demonstrates a
more refined phenomenon: port numbering induces a two-class
partition of entities, and the indistinguishability holds at the
class level rather than globally.  This shows that even augmenting
MPNNs beyond 1-WL via port numbering is insufficient for $K_{2,r}$
detection when $r \geq 2$.  Ego IDs provide the additional
expressive power by breaking this class-level symmetry.

\medskip
\noindent\textbf{Scope of the necessity results.}
The necessity results established in this paper
(Theorems~\ref{thm:K21-simple}, \ref{thm:K21-multigraph},
\ref{thm:K2r-simple}, and~\ref{thm:cycle-detection}) are specific
to the MPNN framework augmented with the three adaptations of
Egressy et al.~\cite{Egressy-et-al:AAAI:2024}: reverse message
passing, directed multigraph port numbering, and ego IDs.  They
show that within this framework, ego IDs are unavoidable for
$K_{2,r}$ detection when $r \geq 2$.  Alternative architectures
that exceed the 1-WL expressivity bound --- such as
$k$-dimensional GNNs corresponding to higher-order $k$-WL
methods~\cite{Morris-et-al:AAAI:2019} or subgraph
GNNs~\cite{Bevilacqua-et-al:ICLR:2022} --- operate outside this
framework and are not subject to the same impossibility results.
In particular, higher-order methods that maintain representations
for node \emph{pairs} have the representational capacity to encode
pairwise relationships across the graph --- the kind of
cross-entity information that, within the MPNN framework, ego IDs
are needed to provide.
The present paper does not analyse these alternative architectures;
our focus is on the MPNN framework, which underlies the majority
of deployed GNN-based entity resolution
systems~\cite{Li:GraphER:AAAI:2020, Yao:HierGAT:SIGMOD:2022,
	Ganesan:IBM:2020}.
	
\medskip
\noindent\textbf{Practical implications for entity resolution
	systems.}
For large-scale entity resolution systems processing millions of
records, our results provide concrete architectural guidance.

In \emph{master data management} systems such as customer databases
and product catalogues, the typical matching criterion is any shared
attribute value --- email address, phone number, or postal address.
Theorem~\ref{thm:K21-simple} guarantees that a reverse message
passing GNN with two layers suffices for this task, with no ego IDs
needed, yielding $O(|E|)$ per-layer computation.  This is the
cheapest possible architecture.

For \emph{high-confidence deduplication} requiring multiple matching
attributes --- the standard in practice, where a single shared phone
number may be coincidental but shared email \emph{and} phone number
is strong evidence ---
Theorem~\ref{thm:K2r-simple} shows that ego IDs are unavoidable.
However, the explicit $4$-layer construction provides a concrete and
efficient architecture, and the $O(|V| \cdot |E|)$ cost is inherent
to any MPNN approach.

For \emph{financial crime detection} involving cyclic transaction
patterns or layered ownership structures,
Theorem~\ref{thm:cycle-detection} shows that ego IDs are necessary
and sufficient on the canonical separation instances.  Universal
cycle detection on arbitrary directed multigraphs requires the full
triple of adaptations
\cite[Corollary~4.4.1]{Egressy-et-al:AAAI:2024}, but for the
functional-graph instances that arise most commonly in regulatory
screening, the lighter architecture suffices.

\medskip
\noindent\textbf{Limitations and future work.}
Several directions remain open.  First, our sufficiency
constructions assume exact injective aggregation; understanding the
impact of finite-precision arithmetic and learned approximate
aggregation functions on the separation thresholds is an important
practical question.  Second, the cycle-detection sufficiency
(Theorem~\ref{thm:cycle-detection}, part~(b)) is established only
for the separation pair $G_1 = 2C_\ell$ and $G_2 = C_{2\ell}$; the
walk-based MPNN constructed in the proof detects closed walks rather
than simple cycles, and these two notions diverge on general graphs
(Remark~\ref{rem:walks-vs-cycles}).  Determining the minimal
adaptation set for universal $\mathrm{Cyc}_\ell$ computation --- in
particular, whether reverse message passing and port numbering are
each independently necessary for $\mathrm{Cyc}_\ell$ on all typed
directed multigraphs --- is left to future work.  Third, we have
focused on the binary co-reference predicate $\mathrm{Dup}_r$;
extending the separation theory to soft matching scores, approximate
attribute similarity under feature perturbations, and higher-order
co-reference patterns such as $K_{n,r}$ detection for $n \geq 3$
would further bridge the gap between theoretical expressivity and
practical entity resolution pipelines.  The computational validation reported in
Section~\ref{sec:computational} confirms every theoretical
prediction on the separation graphs; extending this validation to
large-scale entity resolution benchmarks remains an open direction.

\medskip\noindent\textbf{AI assistance disclosure.}
An AI-assisted research workflow was used throughout the development
of this work. Initial research directions were explored using
multiple large language models. Proof development, refinement, and
simulation experiments were then carried out through extended
iterative collaboration with Anthropic's Claude Opus~4.6 across
multiple sessions. Throughout this process, the author provided the
research formulation, independently constructed counterexamples and
proof components, identified and corrected mathematical errors in
LLM-generated proofs, made design decisions regarding proof
minimality and notation, and verified all final results. The author
takes full responsibility for the correctness of all claims. The
\LaTeX{} typesetting, notational consistency checking, and prose
refinement were subsequently carried out with the assistance of
LLMs. For a recent example of productive
human-AI research collaboration, see Knuth~\cite{Knuth:2026}.

\section{Acknowledgements}
This work was carried out while the author was a Project Consultant
at the Centre for Cybersecurity, Trust and Reliability (CyStar),
Department of Computer Science and Engineering, Indian Institute of
Technology, Madras. The work was supported by a Corporate Social
Responsibility (CSR) grant on master data management from Dun \& Bradstreet to IIT Madras 
(PI: Prof.\ John Augustine). The author thanks Prof.\ Augustine
for his encouragement and support.

 {
\addcontentsline{toc}{section}{References}
\bibliographystyle{plain}
\bibliography{refs_ag.bib}

}
\end{document}